\documentclass{article}

\usepackage{microtype}
\usepackage{graphicx}
\usepackage{subcaption}
\usepackage{booktabs} 

\usepackage{hyperref}
\usepackage{color}
\usepackage{tabularray}
\usepackage{xcolor}
\usepackage{tcolorbox}

\usepackage[accepted]{icml2026}

\usepackage{amsmath}
\usepackage{amssymb}
\usepackage{mathtools}
\usepackage{amsthm}
\usepackage{enumitem}
\usepackage{tikz}
\usetikzlibrary{positioning, arrows.meta, shapes.geometric}

\usepackage[capitalize,noabbrev]{cleveref}

\theoremstyle{plain}
\newtheorem{theorem}{Theorem}[section]
\newtheorem{proposition}[theorem]{Proposition}

\theoremstyle{definition}

\theoremstyle{remark}

\newcommand{\R}{\mathbb{R}}
\newcommand{\E}{\mathbb{E}}

\newcommand{\state}[1][t]{s_{#1}}
\newcommand{\tstate}[1][t]{\tilde{s}_{#1}}
\newcommand{\statenext}{\state[t+1]}
\newcommand{\traj}[1][0:T]{s_{#1}}
\newcommand{\ttraj}[1][0:T]{\tilde{s}_{#1}}
\newcommand{\trajcurr}[1][t]{\traj[0:#1]}
\newcommand{\trajfuture}[1][t]{\traj[#1:T]}
\newcommand{\action}{a_t}
\newcommand{\reward}{r}
\newcommand{\network}{f}
\newcommand{\networktrain}{\network_\param}
\newcommand{\valuefunc}{V}
\newcommand{\valuetrain}{\valuefunc_\param}
\newcommand{\statespace}{\mathcal{S}}
\newcommand{\actionspace}{\mathcal{A}}
\newcommand{\policy}{\pi}
\newcommand{\policyopt}{\pi_*}
\newcommand{\param}{\theta}
\newcommand{\policytrain}{\pi_\param}
\newcommand{\policycond}{\policy(\action \mid \state)}
\newcommand{\KL}{\mathbb{KL}}
\newcommand{\resamplingcrit}{c}
\newcommand{\loss}{\mathcal{L}}

\newcommand{\appref}[1]{Appendix~\ref{#1}}

\DeclareMathOperator*{\argmax}{arg\,max}
\DeclareMathOperator*{\argmin}{arg\,min}

\usepackage[textsize=tiny]{todonotes}

\icmltitlerunning{Agentic Monte Carlo: Simulating Reinforcement Learning for Black-Box Agents}

\begin{document}

\twocolumn[
  \icmltitle{Agentic Monte Carlo: Simulating Reinforcement Learning for Black-Box Agents}



  \icmlsetsymbol{equal}{*}

  \begin{icmlauthorlist}
    \icmlauthor{Dae Yon Hwang}{L6}
    \icmlauthor{Raunaq Suri}{L6}
    \icmlauthor{Valentin Villecroze}{L6}
    \icmlauthor{Anthony L. Caterini}{L6}\\
    \icmlauthor{Jesse C. Cresswell}{L6}
    \icmlauthor{No\"el Vouitsis}{L6}
    \icmlauthor{Brendan Leigh Ross}{L6}

  \end{icmlauthorlist}

  \icmlaffiliation{L6}{Layer 6 AI, Toronto, Canada}

  \icmlcorrespondingauthor{Dae Yon Hwang}{daeyon@layer6.ai}
  \icmlcorrespondingauthor{Brendan Leigh Ross}{brendan @layer6.ai}

  \icmlkeywords{Agents, Large Language Models, Reinforcement Learning, Sequential Monte Carlo}

  \vskip 0.3in
]



\printAffiliationsAndNotice{}  

\begin{abstract}
  
LLM agents operate in two distinct regimes: open-weight agents amenable to reinforcement learning (RL) and black-box agents whose behaviour must be controlled purely at test time. Although black-box agents are often backed by state-of-the-art proprietary LLMs, API-only access precludes parameter-level optimization, rendering most RL methods inapplicable. To address this limitation, we turn to a known equivalence between RL and Bayesian inference. We propose Agentic Monte Carlo (AMC) to directly sample from the optimal policy of a black-box agent rather than training it through RL. The optimal policy is a posterior over trajectories whose prior we define as the fixed black-box LLM agent. We employ Sequential Monte Carlo to sample from this posterior by learning a value function to steer the agent while leaving the underlying black-box model unchanged. We validate AMC on three diverse environments from the AgentGym benchmark, demonstrating significant improvements over prompting baselines and even outperforming Group Relative Policy Optimization (GRPO) as we scale the test-time compute of our method. AMC demonstrates the feasibility of performing principled RL-style optimization of black-box LLM agents.\footnote{Code is available at {\url{https://github.com/layer6ai-labs/Agentic-Monte-Carlo}}}

\end{abstract}

\section{Introduction}\label{sec:intro}

\begin{figure}[htbp]
    \centering
    \includegraphics[width=0.48\textwidth]{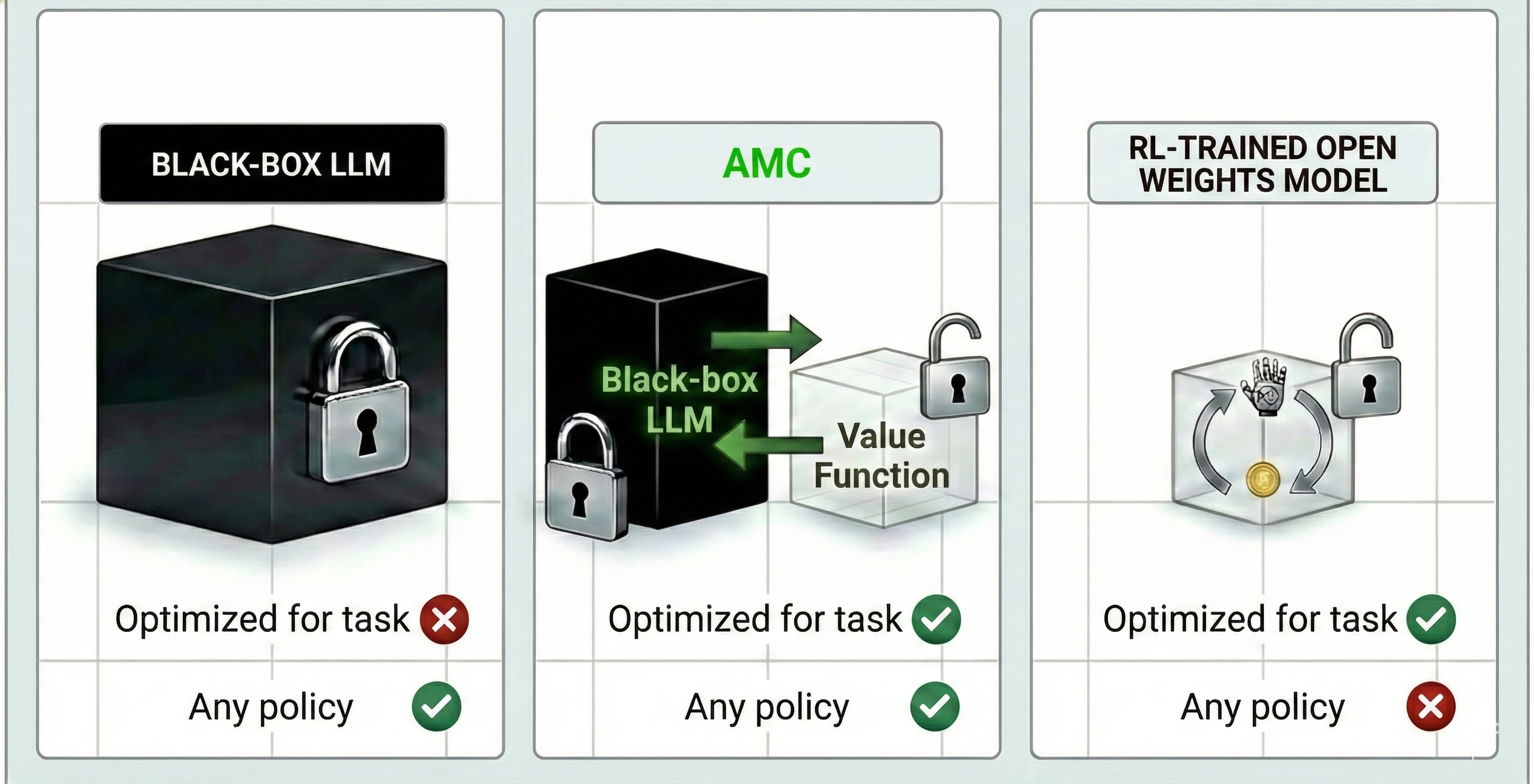}
    \caption{Comparison between AMC and other agentic paradigms. AMC facilitates task-specificity for black-box LLM policies using a learned lightweight value function. On the other hand, training a white-box model with RL imposes constraints on the choice and scale of the base policy.}
    \label{fig:amc_visual_comparison}
    \vspace{-8pt}
\end{figure}

Large Language Models (LLMs) have evolved from passive text generators to active autonomous agents capable of multi-step reasoning and environment interaction \cite{yao2022react, schick2023toolformer}. To operationalize these agents, Reinforcement Learning (RL) \cite{sutton1998reinforcement, christiano2017deep, ziegler2019fine} has become the dominant training paradigm through methods like Proximal Policy Optimization (PPO) \cite{schulman2017proximal, ouyang2022training} and, more recently, Group Relative Policy Optimization (GRPO) \cite{shao2024deepseekmathpushinglimitsmathematical}. These algorithms have proven highly effective for open-weight models, enabling abilities ranging from complex mathematical reasoning to software engineering \cite{wang2025reinforcement, pennino2025reasoning}. However, these approaches rely on a fundamental assumption: white-box access to model parameters to compute policy gradients. This requirement presents a fundamental barrier for the most capable state-of-the-art models—such as GPT-5 \cite{openai2025gpt5}, Gemini 3 \cite{deepmind_gemini3pro_2025}, and Claude 4.6 \cite{anthropic_claude_opus46_2026}—which are predominantly available only as black-box APIs. Consequently, researchers seeking to optimize agents backed by these proprietary models are often restricted to prompt engineering or fine-tuning open-weight surrogates, neither of which performs RL-based optimization of the target black-box model.

To bridge this gap, we revisit the known duality between RL and Bayesian inference \citep{levine2018reinforcement,korbak2022rl}. This framework treats the goal of maximizing rewards using RL as equivalent to probabilistic inference of the agent's optimal policy. Here, the optimal policy is defined as a Bayesian posterior distribution over agent trajectories. Specifically, this posterior is proportional to a prior distribution—given by the pre-trained black-box model—multiplied by a likelihood term representing the probability of achieving high rewards (i.e., of being optimal). This formulation is particularly attractive for black-box agents because it reframes the learning problem: instead of updating the parameters of the prior (which we cannot access), we can simply sample from the posterior distribution to recover optimal behaviour. However, exact sampling from this optimal posterior is impossible for black-box agents, whose full log-probabilities are generally unavailable, and intractable even for white-box agents because of the high-dimensional action space and long time horizons inherent in agentic environments.

In this work, we propose Agentic Monte Carlo (AMC), a novel framework that makes this sampling tractable by leveraging Sequential Monte Carlo (SMC) methods—a form of importance sampling applied to sequential data \cite{Doucet2001}. At each step of the agent's trajectory, AMC samples actions directly from the black-box prior and re-weights them based on expected rewards. To predict expected rewards, we train a separate value function to effectively steer the agent toward optimality without ever modifying the underlying black-box language model. As we highlight in \autoref{fig:amc_visual_comparison}, AMC provides a useful interpolation between prompting a static black-box agent and fine-tuning a smaller white-box one.

We validate AMC using the AgentGym benchmark \cite{agentgym}, evaluating performance in three diverse environments: WebShop (e-commerce), SciWorld (scientific experiments), and TextCraft (Minecraft-style crafting). Our experiments demonstrate that AMC consistently outperforms prompting baselines and, as we scale the test-time compute of our method, even outperforms GRPO baselines that require full parameter access.

Our key contributions are as follows:
\begin{enumerate}
    \item We formulate the problem of RL for black-box agents through the lens of Bayesian inference, defining the optimal policy as a Bayesian posterior over trajectories that combines the black-box prior with a likelihood term that encodes optimal agent behaviour.

    \item We propose Agentic Monte Carlo, an algorithm that uses SMC with sequential importance resampling to sample from this optimal policy. We derive a weight update rule that leverages a learned value function to guide the black-box agent's behaviour toward optimality.

    \item We empirically validate the effectiveness of our approach compared to both prompting and RL baselines, demonstrating that rigorous RL concepts can be applied to closed-source agents without gradient access. These results position AMC with black-box models as a viable alternative to gradient-based RL in GPU-constrained settings.
\end{enumerate}

\section{Agentic Monte Carlo}
\label{sec:method}
This section builds up to our method, starting from preliminaries around reinforcement learning. In \autoref{sec:background_rl}, we introduce the reinforcement learning setting and formulate it as a Bayesian inference problem. We outline how sequential Monte Carlo (SMC) can be used as a tool to solve this problem in \autoref{sec:bayesian_inference}. Then, in \autoref{sec:learning_value} and \autoref{sec:amc_method}, we describe Agentic Monte Carlo (AMC): our approach for adapting SMC to potentially black-box agents by learning an ancillary value function.

\subsection{Reinforcement Learning \& Bayesian Inference}
\label{sec:background_rl}

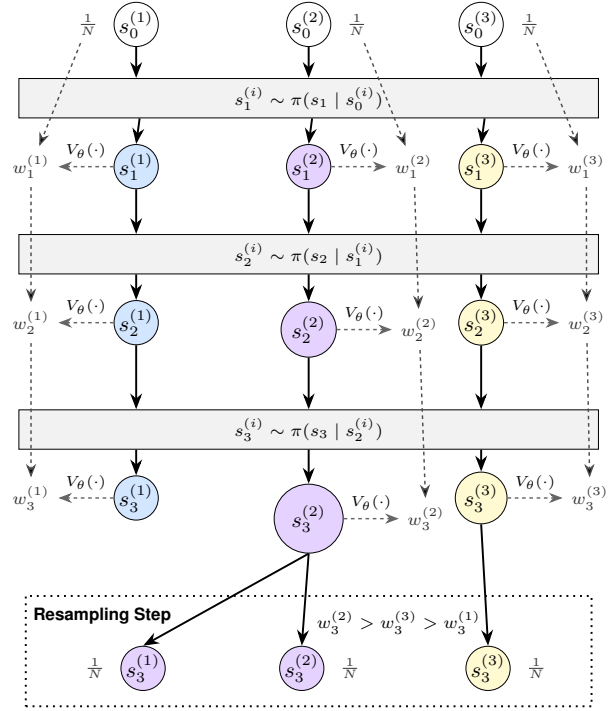
\begin{figure}[t]
\centering
\definecolor{trajBlue}{RGB}{210, 230, 255}
\definecolor{trajPurple}{RGB}{230, 210, 255}
\definecolor{trajYellow}{RGB}{255, 250, 210}
\resizebox{\columnwidth}{!}{
\begin{tikzpicture}[
    node distance=0.4cm and 0.2cm,
    state/.style={circle, draw, inner sep=0pt, font=\small}, 
    vbox/.style={rectangle, draw, minimum width=8.4cm, minimum height=0.4cm, font=\scriptsize\sffamily, 
                 fill=gray!10},
    arrow/.style={-Stealth, thick},
    v_arrow/.style={-Stealth, dash pattern=on 1.5pt off 1.5pt, semithick, draw=black!60},
    history_line/.style={-Stealth, semithick, dash pattern=on 1.5pt off 1.5pt, draw=black!70}, 
    w/.style={font=\scriptsize, color=black!80} 
]

    \node[state, fill=white, minimum size=0.55cm] (s0_1) at (-2.7, 0) {$s_0^{(1)}$};
    \node[state, fill=white, minimum size=0.55cm] (s0_2) at (-0.2, 0) {$s_0^{(2)}$};
    \node[state, fill=white, minimum size=0.55cm] (s0_3) at (2.3, 0) {$s_0^{(3)}$};
    
    \node[w, left=0.1cm of s0_1] (w0_1) {$\frac{1}{N}$};
    \node[w, right=0.1cm of s0_2] (w0_2) {$\frac{1}{N}$};
    \node[w, right=0.1cm of s0_3] (w0_3) {$\frac{1}{N}$};

    \node[vbox, below=0.45cm of s0_2] (v0) {$s_1^{(i)} \sim \pi(s_1 \mid s_0^{(i)})$};
    \draw[arrow] (s0_1) -- ([xshift=-2.5cm]v0.north);
    \draw[arrow] (s0_2) -- (v0.north);
    \draw[arrow] (s0_3) -- ([xshift=2.5cm]v0.north);

    \node[state, fill=trajBlue, minimum size=0.55cm] (s1_1) at ([xshift=-2.5cm, yshift=-0.7cm]v0.south) {$s_1^{(1)}$};
    \node[state, fill=trajPurple, minimum size=0.55cm] (s1_2) at ([xshift=0cm, yshift=-0.7cm]v0.south) {$s_1^{(2)}$};
    \node[state, fill=trajYellow, minimum size=0.55cm] (s1_3) at ([xshift=2.5cm, yshift=-0.7cm]v0.south) {$s_1^{(3)}$};

    \draw[arrow] ([xshift=-2.45cm]v0.south) -- (s1_1.north);
    \draw[arrow] ([xshift=0.05cm]v0.south) -- (s1_2.north);
    \draw[arrow] ([xshift=2.55cm]v0.south) -- (s1_3.north);

    \draw[v_arrow] (s1_1.west) -- ++(-0.8,0) node[left, w] (w1_1) {$w_1^{(1)}$};
    \node[font=\tiny, anchor=south] at ([xshift=-0.4cm]s1_1.west) {$V_{\theta}(\cdot)$};
    \path[history_line] (w0_1) edge (w1_1);

    \draw[v_arrow] (s1_2.east) -- ++(0.8,0) node[right, w] (w1_2) {$w_1^{(2)}$};
    \node[font=\tiny, anchor=south] at ([xshift=0.4cm]s1_2.east) {$V_{\theta}(\cdot)$};
    \path[history_line] (w0_2) edge (w1_2);

    \draw[v_arrow] (s1_3.east) -- ++(0.8,0) node[right, w] (w1_3) {$w_1^{(3)}$};
    \node[font=\tiny, anchor=south] at ([xshift=0.4cm]s1_3.east) {$V_{\theta}(\cdot)$};
    \path[history_line] (w0_3) edge (w1_3);

    \node[vbox, below=0.65cm of s1_2] (v1) {$s_2^{(i)} \sim \pi(s_2 \mid s_1^{(i)})$};
    \draw[arrow] (s1_1) -- ([xshift=-2.5cm]v1.north);
    \draw[arrow] (s1_2) -- (v1.north);
    \draw[arrow] (s1_3) -- ([xshift=2.5cm]v1.north);

    \node[state, fill=trajBlue, minimum size=0.5cm] (s2_1) at ([xshift=-2.5cm, yshift=-0.7cm]v1.south) {$s_2^{(1)}$};
    \node[state, fill=trajPurple, minimum size=0.8cm] (s2_2) at ([yshift=-0.8cm]v1.south) {$s_2^{(2)}$};
    \node[state, fill=trajYellow, minimum size=0.65cm] (s2_3) at ([xshift=2.5cm, yshift=-0.7cm]v1.south) {$s_2^{(3)}$};

    \draw[arrow] ([xshift=-2.5cm]v1.south) -- (s2_1.north);
    \draw[arrow] (v1.south) -- (s2_2.north);
    \draw[arrow] ([xshift=2.5cm]v1.south) -- (s2_3.north);

    \draw[v_arrow] (s2_1.west) -- ++(-0.8,0) node[left, w] (w2_1) {$w_2^{(1)}$};
    \node[font=\tiny, anchor=south] at ([xshift=-0.4cm]s2_1.west) {$V_{\theta}(\cdot)$};
    \path[history_line] (w1_1) edge (w2_1);

    \draw[v_arrow] (s2_2.east) -- ++(0.8,0) node[right, w] (w2_2) {$w_2^{(2)}$};
    \node[font=\tiny, anchor=south] at ([xshift=0.4cm]s2_2.east) {$V_{\theta}(\cdot)$};
    \path[history_line] (w1_2) edge (w2_2);

    \draw[v_arrow] (s2_3.east) -- ++(0.8,0) node[right, w] (w2_3) {$w_2^{(3)}$};
    \node[font=\tiny, anchor=south] at ([xshift=0.4cm]s2_3.east) {$V_{\theta}(\cdot)$};
    \path[history_line] (w1_3) edge (w2_3);

    \node[vbox, below=0.75cm of s2_2] (v2) {$s_3^{(i)} \sim \pi(s_3 \mid s_2^{(i)})$};
    \draw[arrow] (s2_1) -- ([xshift=-2.5cm]v2.north);
    \draw[arrow] (s2_2) -- (v2.north);
    \draw[arrow] (s2_3) -- ([xshift=2.5cm]v2.north);

    \node[state, fill=trajBlue, minimum size=0.45cm] (s3_1) at ([xshift=-2.5cm, yshift=-0.7cm]v2.south) {$s_3^{(1)}$};
    \node[state, fill=trajPurple, minimum size=1.0cm] (s3_2) at ([yshift=-1.0cm]v2.south) {$s_3^{(2)}$};
    \node[state, fill=trajYellow, minimum size=0.75cm] (s3_3) at ([xshift=2.5cm, yshift=-0.7cm]v2.south) {$s_3^{(3)}$};

    \draw[arrow] ([xshift=-2.5cm]v2.south) -- (s3_1.north);
    \draw[arrow] (v2.south) -- (s3_2.north);
    \draw[arrow] ([xshift=2.5cm]v2.south) -- (s3_3.north);

    \draw[v_arrow] (s3_1.west) -- ++(-0.8,0) node[left, w] (w3_1) {$w_3^{(1)}$};
    \node[font=\tiny, anchor=south] at ([xshift=-0.4cm]s3_1.west) {$V_{\theta}(\cdot)$};
    \path[history_line] (w2_1) edge (w3_1);

    \draw[v_arrow] (s3_2.east) -- ++(0.8,0) node[right, w] (w3_2) {$w_3^{(2)}$};
    \node[font=\tiny, anchor=south] at ([xshift=0.4cm]s3_2.east) {$V_{\theta}(\cdot)$};
    \path[history_line] (w2_2) edge (w3_2);

    \draw[v_arrow] (s3_3.east) -- ++(0.8,0) node[right, w] (w3_3) {$w_3^{(3)}$};
    \node[font=\tiny, anchor=south] at ([xshift=0.4cm]s3_3.east) {$V_{\theta}(\cdot)$};
    \path[history_line] (w2_3) edge (w3_3);

    \node[draw, rectangle, dotted, thick, minimum width=8.2cm, minimum height=1.6cm, below=0.6cm of s3_2] (resample_box) {};
    
    \node[font=\small\sffamily\bfseries, anchor=north west] at ([yshift=-0.05cm, xshift=0cm]resample_box.north west) {\scriptsize Resampling Step};
    \node[font=\scriptsize, anchor=north east] at ([yshift=-0.1cm, xshift=-1.5cm]resample_box.north east) {$w_3^{(2)} > w_3^{(3)} > w_3^{(1)}$};

    \node[state, fill=trajPurple, minimum size=0.6cm] (r1) at ([xshift=-2.4cm, yshift=-0.25cm]resample_box.center) {$s_3^{(1)}$};
    \node[state, fill=trajPurple, minimum size=0.6cm] (r2) at ([xshift=-0.1cm, yshift=-0.25cm]resample_box.center) {$s_3^{(2)}$};
    \node[state, fill=trajYellow, minimum size=0.6cm] (r3) at ([xshift=2.6cm, yshift=-0.25cm]resample_box.center) {$s_3^{(3)}$};

    \draw[arrow] (s3_2.south) -- (r1.north);
    \draw[arrow] (s3_2.south) -- (r2.north);
    \draw[arrow] (s3_3.south) -- (r3.north);

    \node[w, left=0.1cm of r1] (rf_1) {$\frac{1}{N}$};
    \node[w, right=0.1cm of r2] (rf_2) {$\frac{1}{N}$};
    \node[w, right=0.1cm of r3] (rf_3) {$\frac{1}{N}$};

\end{tikzpicture}
}
\caption{
\looseness=-1 Visual representation of AMC for $N=3$ trajectories. Importance weights $w_t^{(i)}$ are determined using the value function $V_{\theta}$, where lower-weighted trajectories (e.g., $s^{(1)}$) are more likely to be pruned than higher-weighted ones ($s^{(2)}, s^{(3)}$) during resampling.}
\label{fig:amc}
\vspace{-2em}
\end{figure}

We begin our journey in a standard reinforcement learning setting in which an agent, defined by a policy $\policy$, responds to its observed state $\state \in \statespace$ at each timestep $t$ by producing an action $\action \sim \policycond$, $\action \in \actionspace$. The selected action results in a new state $\state[t+1] \sim p(\state[t+1] \mid \state, \action)$ determined by the transition dynamics of the agent's environment. We assume discrete state and action spaces. Preferences over states are encoded by a scalar reward function $\reward(\state) \in \R$ given by the environment, where higher rewards are preferred.

Over a time horizon $T$, the agent $\policy$ produces a trajectory $\traj \sim \policy(\traj)$.\footnote{Since our policy of interest $\policy$ will be notionally intractable, we fold transition dynamics into our concept of the policy, ignore actions, and focus on the resulting state transitions $\policy(\statenext \mid \state)$ throughout.} An optimal agent should produce trajectories that have large cumulative reward values denoted by the shorthand $\reward(\traj) = \sum_{t=0}^T \reward(\state)$.\footnote{We omit discount factors here.}

For many tasks and environments of interest, it has become standard to initialize policies using pre-trained LLMs given their rich semantic understanding and broad world knowledge. Our goal is therefore to fine-tune an LLM agent for a particular environment. Policy gradient methods \cite{williams1992simple} have become a dominant training paradigm with LLMs: the policy is trained to maximize expected rewards, usually with a KL-divergence penalty term against a reference policy \cite{ouyang2022training, shao2024deepseekmathpushinglimitsmathematical},
\begin{equation}\label{eq:variational_inference}
    \policyopt = \argmax_{\policytrain} \, \E_{\policytrain(\traj)} \!\left[\reward(\traj)\right]\! - \!\beta \KL\left[\policytrain\, \Vert\,  \policy\right],
\end{equation}
where $\policytrain$ is the policy being trained, $\policy$ is the reference model given by a pre-trained LLM, and $\beta$ is a regularization coefficient. Although $\policytrain$ is usually initialized from $\policy$, the subscript $\param$ identifies $\policytrain$ as being a separate model with trainable parameters, while $\policy$ itself remains fixed.

\looseness=-1
As highlighted by \citet{korbak2022rl} (and for completeness, proven in \appref{app:theory_equivalence}), \autoref{eq:variational_inference} can be understood as performing variational inference to approximate the following (intractable) posterior:
\begin{equation}\label{eq:bayesian_inference}
    \policyopt(\traj) \propto \policy(\traj) e^{\reward(\traj)/\beta}.
\end{equation}
\autoref{eq:bayesian_inference} frames KL-regularized RL as a Bayesian inference problem where the reference policy's probabilities $\policy(\traj)$—which serve as the prior—are modulated to upsample high-reward trajectories as encoded by the likelihood $e^{\reward(\traj)/\beta}$. Note that evaluating the right-hand sides of \autoref{eq:variational_inference} and \autoref{eq:bayesian_inference} yields the same optimal policy: the posterior $\policyopt$.

By framing KL-regularized RL as Bayesian inference, we can look beyond variational inference toward the broader Bayesian toolbox. While variational inference (the standard approach) estimates the posterior by training a parameterized policy $\policytrain$, Monte Carlo methods offer a sampling-based alternative that could bypass policy optimization entirely, allowing us to simulate reinforcement learning even when $\policy$ is a black box. We exploit this shift in perspective to show that $\policyopt$ can be estimated by guiding a static prior policy $\policy$ with a simple auxiliary model. This approach provides a viable pathway for reinforcement learning in scenarios where fine-tuning the reference policy $\policy$ is intractable or impossible.

\subsection{Bayesian Inference via Sequential Monte Carlo}\label{sec:bayesian_inference}

Sequential Monte Carlo \citep[SMC;][]{gordon1993novel, Doucet2001} is a group of Monte Carlo methods for estimating posteriors of sequentially evolving systems conditioned on step-wise observations. The sequential nature of our state transitions $\policy(\state[t+1] \mid \state)$ makes SMC particularly well-suited to the RL setup described in \autoref{sec:background_rl}. For example, \citet{piche2018probabilistic} use SMC to perform planning with a trained policy and value function, and \citet{zhao2024probabilistic} use SMC along with a learned ``twist'' function for constrained text generation using LLMs. Since the nature of many digital environments allows us to simulate agent trajectories in parallel, we can use SMC to directly structure a Monte Carlo simulation of an agent running the optimal policy $\policyopt$.

\looseness=-1
The most standard SMC method is the bootstrap filter \citep{gordon1993novel,Doucet2001}, i.e., sequential importance resampling (SIR). At a high level, SIR provides a principled way to sample several trajectories $\{\traj^{(i)}\}_{i=1}^N$ in parallel by running our  (black-box or otherwise intractable) prior $\policy$, while intermittently \emph{resampling}: stochastically pruning the worst and proliferating the best trajectories in order to steer the overall sample towards the posterior $\policyopt$. Fundamentally, SIR is an importance sampling algorithm, and importance weights will be used to guide these resampling steps.

\looseness=-1
Consider a \emph{partial} agent trajectory $\trajcurr$ up to timestep $t$. Since it is intractable to sample trajectories from $\policyopt$, importance sampling instead samples from a tractable \emph{proposal} policy and then corrects for the policy mismatch using corresponding importance weights. We know that it \emph{is} tractable to sample from our (fixed) LLM prior $\pi$, so we can sample trajectories $\{\trajcurr^{(i)}\}_{i=1}^N \sim \policy$ and then reweight them with importance weights proportional to the ratio of their likelihoods under both \(\policyopt\) and $\pi$:
\begin{equation}\label{eq:importance_weight_def}
  w_t = \frac{\policyopt(\trajcurr)}{\policy(\trajcurr)}.
\end{equation}
In order to use these importance weights for resampling, SIR requires a recursive decomposition of $w_t$ in terms of the previous weight, $w_{t-1}$. In \appref{app:theory_decomp}, we show that
\begin{align}
\begin{aligned}\label{eq:importance_weight_recursive}
  w_t &= w_{t-1}\cdot \frac{\E_{\policy(\trajfuture[t+1] \mid \state)}[e^{\frac{1}{\beta}\sum_{\tau=t}^T \reward(\state[\tau])}]}
  {\E_{\policy(\trajfuture \mid \state[t-1])}[e^{\frac{1}{\beta}\sum_{\tau=t}^T \reward(\state[\tau])}]},\\
  &= w_{t-1} \cdot e^{\valuefunc(\state) - \valuefunc(\state[t-1]) + \frac{\reward(\state[t-1])}{\beta}},
\end{aligned}
\end{align}
where we define $\valuefunc(\state) = \log\E_{\policy(\trajfuture[t+1] \mid \state)}[e^{\frac{1}{\beta}\sum_{\tau=t}^T \reward(\state[\tau])}]$. $\valuefunc$ is known as as the \emph{soft value function} in maximum entropy RL because its log-sum-exp structure emulates a soft maximum reward value over trajectories \citep{haarnoja2017reinforcement,haarnoja2018soft}. We refer the reader to \citet{levine2018reinforcement} for more background. For now, let us assume that we can compute $\valuefunc$. 
This recursive decomposition allows us to perform SIR \citep{Doucet2001}, which we depict visually in \autoref{fig:amc}. The full algorithm, \autoref{alg:sir}, is relegated to \appref{app:sir_alg} due to space constraints.

\looseness=-1
SIR runs $N$ trajectories in parallel and, at each step $t$, updates their respective importance weights $w_t^{(i)}$. At certain timesteps, we \textit{resample}, with replacement, a new subset of trajectories $\{\trajcurr^{(i)}\}_{i=1}^N$ according to their corresponding importance weights $\{w_t^{(i)}\}_{i=1}^N$, and then reset their importance weights to be uniform. The decision of whether to resample at a given timestep is made by an arbitrary boolean resampling criterion $\resamplingcrit(t, \{w_t^{(i)}\}_{i=1}^N)$, which in experiments we simply set to \texttt{true} only on a cross-validated subset of timesteps. Resampling criteria are further explored in \appref{app:resampling}.

Upon completion of SIR, the weighted empirical measure of the resulting trajectories $\{\traj^{(i)}\}_{i=1}^N$ and weights $\{w_T^{(i)}\}_{i=1}^N$ converges weakly to $\policyopt$ as $N \to \infty$ \cite{Doucet2001}. Put more simply, sampling from the categorical distribution induced by the normalized weights approximates a sample from $\policyopt$, and does so better as $N$ gets larger.

Until now, we assumed that we could compute the value function $\valuefunc$, which is an expectation involving future rewards $\{\reward(\state[\tau])\}_{\tau=t}^T$. Estimating this expectation precisely would require multiple simulations of our agent until the terminal timestep $T$, which can quickly become computationally prohibitive. Instead, we opt to \emph{learn} $\valuefunc$, a procedure we describe in the following section.

\subsection{Learning Importance Weights for Agentic Systems}\label{sec:learning_value}

We now shift our focus to estimating the value function $\valuefunc$ to then efficiently compute importance weights when running SIR. We aim to learn a model $\valuetrain$ to approximate $\valuefunc$. Because $\policy$ is a frozen prior, we can regress our value function directly on a set of Monte Carlo trajectories from $\policy$, optimizing for the true value of $\valuefunc$ by way of the following relationship. We assume there exists a set of weights for $\valuetrain$ that gives rise to $\valuefunc$ as
\begin{equation}\label{eq:loss_func_ideal}\small
    \valuefunc{=} \argmin_{\valuetrain}\E_{p(t, \state)}\!\left\lVert \valuetrain(\state){-}\log \E_{\policy(\trajfuture[t+1] \mid \state)} \!\!\left[e^{\frac{1}{\beta} \sum_{\tau=t}^T{r}(\state[\tau])}\right]\!\right\rVert_2^2, \\
\end{equation}
where $p(t, \state)$, the distribution of training inputs, is any distribution with full support over timesteps $t$ and states $s_t$. We can then optimize using gradient-based methods. 

\looseness=-1
In our setting, $\reward(\state)$ is available at test time, so we only need to learn future rewards. We thus parameterize $\valuetrain$ as
\begin{equation}\label{eq:value_parameterization}
    \valuetrain(\state) = \networktrain(\state) + \reward(\state),
\end{equation}
where $\networktrain$ is a transformer-based model with a regression head that outputs a scalar prediction of $\reward(s_{t+1:T})$ (if $\reward$ were not available at test-time, we would parameterize $\valuetrain$ directly). In practice, we estimate the outer expectation by first sampling $M$ trajectories $\{\traj^{(j)}\}_{j=1}^M$ from $\policy$ and, from these trajectories, collating a subset of individual states and corresponding timesteps to use as our training set $\{(t_k, \state[t_k]^{(k)})\}_{k=1}^P$. We approximate the inner expectation by a single trajectory, thus accepting some bias through the logarithm to arrive at an estimator for the (non-soft) value function as the regression target:
\begin{equation}\label{eq:loss_func_real}
    \loss(\networktrain)\! =\!\frac{1}{P}\sum_{k=1}^P\!\left\lVert \networktrain(\state[t_k]^{(k)}) -\sum_{\tau=t_k+1}^T\reward(\state[\tau]^{(k)})\right\rVert_2^2, \\
\end{equation}
where we set $\beta=1$ for training and adjust the sampling temperature post hoc. It is with this trained network that we model $\valuetrain$, which then estimates the weight updates needed for SIR in \autoref{alg:sir}. In \appref{app:value_estimation}, we show that training with \autoref{eq:loss_func_real} performs well in comparison to the more expensive and complex \autoref{eq:loss_func_ideal}.

\subsection{Agentic Monte Carlo}\label{sec:amc_method}
To summarize our method—which we refer to as Agentic Monte Carlo (AMC)—we begin with a black-box (or otherwise intractable) LLM policy $\policy$ that forms a presumably good prior for a given task and environment of interest. This is common, for example, with today's state-of-the-art, generalist models accessible only through APIs (e.g., GPT-5, Gemini 3, Claude 4.6), for which we are limited to prompting methods to steer the agent's behaviour.

Ideally, we would like to be able to fine-tune these agents using standard RL algorithms to estimate $\policyopt$ from \autoref{eq:variational_inference}, but this feat is impossible due to their black-box nature. Instead, AMC turns to probabilistic inference to approximate $\policyopt$ using SMC.

Before running SMC, we must build our value function. We collect a series of trial trajectories from the prior $\{\traj^{(j)}\}_{j=1}^M \sim\policy(\traj)$ and use these trajectories to train $\networktrain$ according to \autoref{eq:loss_func_real}. We parameterize our value function $\valuetrain$ according to \autoref{eq:value_parameterization}.

We next use SMC, depicted in \autoref{fig:amc}, to generate an approximate set of samples from $\policyopt$ using \autoref{alg:sir}, where we substitute the true value function $\valuefunc$ with our learned value function $\valuetrain$ in the importance weight computation step. This gives us a final set of trajectories $\{\traj^{(i)}\}_{i=1}^N$ and weights $\{w_T^{(i)}\}_{i=1}^N$ that together approximate $\policyopt$. In practice—and in our experiments—we may only be interested in selecting a single best trajectory. We can simply keep the trajectory with the highest cumulative reward or, if $\reward$ is unavailable at test time, the trajectory $i$ with the largest  weight $w_T^{(i)}$.

\section{Related Work}
\label{sec:related-work}

\noindent \textbf{Bayesian Inference, Control-as-Inference, and SMC.} The connections we exploit between KL-regularized RL and Bayesian inference \citep{korbak2022rl} are grounded in the control-as-inference framework \cite{dayan1997using, toussaint2006probabilistic, levine2018reinforcement}, which casts optimal decision-making as probabilistic inference. \citet{piche2018probabilistic} pioneered the use of SMC for planning in this domain, training policy proposals to approximate the optimal posterior. Building on this, \citet{lioutas2023criticsequentialmontecarlo} introduced CriticSMC, identifying that standard SMC suffers from sample impoverishment in sparse-reward settings; they proposed weighting particles using a learned soft Q-function to steer sampling toward high-value regions. Although AMC shares theoretical roots with these two works, they are designed for continuous control tasks and deterministic transition dynamics, respectively. AMC adapts these principles to the discrete, semantic state space of black-box LLM agents, where gradients are inaccessible and state transitions are defined by opaque—possibly stochastic—tool interactions, rather than physics engines.

\noindent \textbf{SMC for Language Models.} Recent research has begun applying SMC to steer LLM generation toward desired constraints \cite{lew2023sequential, zhao2024probabilistic, loula2025syntactic}. However, these methods modulate the proposal distribution via access to the LLM's logits, a requirement that is impossible to satisfy with proprietary black-box models; in contrast, AMC is designed specifically for black-box settings. Furthermore, while these works focus on constrained text generation, AMC targets sequential agentic workflows involving external tool interaction. Closer to our work is ``Rollout Roulette'' by \citet{puri2025rollout}, which applies SMC using pre-trained Process Reward Models (PRMs) \cite{lightman2023let} to scale test-time reasoning. Instead, we focus on multi-step, interactive agents where the state includes external environment observations, necessitating a critic trained on interaction history rather than a standard reasoning PRM. Furthermore, AMC integrates rigorous RL concepts, viewing the critic not just as a heuristic verifier but as a learned proxy for the (soft) value function in a control-as-inference formulation.

\looseness=-1
\noindent \textbf{Search and Planning for LLM Agents.} The limitations of greedy prompting strategies for LLM agents like ReAct \cite{yao2022react} have spurred interest in inference-time search. Tree-based methods such as Language Agent Tree Search \cite{zhou2024language} and ExACT \cite{yuexact} adapt Monte Carlo Tree Search (MCTS) \cite{silver2017mastering} to LLMs but often suffer from high latency due to sequential node expansion. Population-based methods offer a parallelizable alternative. \citet{klein2025fleetagentscoordinatedproblem} introduced Fleet of Agents (FoA), which employs a genetic-style particle filter to coordinate agent problem-solving. However, it relies on static, heuristic value functions (e.g., lexical overlap or prompting). Conversely, AMC integrates rigorous RL by training a parameterized value function to estimate the likelihood of agent success, allowing our method to learn optimality from data rather than relying on hand-crafted heuristics.

\section{Experiments}
\label{sec:experiments}

We first detail our experimental configuration, and follow that by presenting overall results demonstrating AMC's effectiveness across benchmarks compared to baselines. We then provide a multi-perspective analysis, including a head-to-head comparison with policy gradient RL, an ablation study on value function optimization, and a cost-efficiency evaluation. 
Additional analyses, including the impact of the value function's base model, resampling strategies, metrics, and additional datasets, can be found in \autoref{app:experiments}. Qualitative examples are provided in \autoref{app:qualitative_analysis}.

\subsection{Experimental Configuration}
Here we describe our main experimental settings. We relegate additional details to \autoref{app:setup} and \autoref{app:templates}.

\noindent \textbf{Datasets.} Our core evaluations are done across three agentic benchmarks: web-based commerce  \citep[WebShop;][]{yao2023webshopscalablerealworldweb}, scientific reasoning \citep[SciWorld;][]{wang2022scienceworldagentsmarter5th}, and hierarchical crafting \citep[TextCraft;][]{prasad2024adaptasneededdecompositionplanning}. 
To demonstrate the generalizability of our approach, we provide further, albeit limited, evaluation on the Weather and Movie datasets \cite{ma2024agentboardanalyticalevaluationboard}, relegating these to \autoref{tab:weather_movie_result} in \appref{app:additional_datasets}. These environments encompass both dense and sparse reward structures.
We evaluate our agents using AgentGym \cite{agentgym}, a centralized platform comprising a heterogeneous collection of environments that span a diverse range of real-world tasks.

\noindent \textbf{Training.} To train our value functions, we curate a set of trajectories $\traj \sim \policy(\traj)$ on a set of training tasks for each benchmark. We initialize $\networktrain$ using pre-trained open-weight language models and then fine-tune, in tandem, a regression head and a set of low-rank adaptation (LoRA) blocks \cite{hu2021loralowrankadaptationlarge}. Further details on the training configuration are available in \autoref{app:setup}.

\noindent \textbf{Base Models.} Consistent with previous work \cite{klein2025fleetagentscoordinatedproblem, agentgym, agentgym-rl}, we initialize our learned network $\networktrain$ from the following base LLMs (all of them the instruct variants): Llama-3.2-11B, Llama-3.1-8B \cite{grattafiori2024llama}, Qwen-2.5-3B, and Qwen-2.5-7B \cite{qwen2025qwen25technicalreport}.  
For our prior LLM policy $\pi$, we use a broad range of models with varying accessibility levels to demonstrate the generalizability of our approach. Specifically, we evaluated open-weight LLM agents—Llama-3.2-11B and Llama-3.1-8B—and black-box agents—GPT-4.1-mini \cite{openai2025gpt41} and GPT-5.1 \cite{openai2025gpt5}.\footnote{We do not use native reasoning capabilities of OpenAI models.} We once again highlight the compatibility of our method with black-box agents as a key feature for many real-world settings.

\noindent \textbf{Baselines.} We compare AMC with two primary baselines performing test-time sampling of parallel agent trajectories: Best-of-$N$ and SMC with Fleet-of-Agents (FoA) prompting \cite{klein2025fleetagentscoordinatedproblem}. 
Best-of-$N$ acts as a baseline, selecting the highest-reward sequence from $N$ parallel trajectories. 
AMC extends this by using learned value estimates from our optimized value function to guide selection throughout the trajectory, rather than relying solely on terminal outcomes. Our SMC baseline with FoA prompting \cite{klein2025fleetagentscoordinatedproblem}, referred to as SMC (FoA), implements trajectory pruning via SMC, leveraging its 2-shot prompt template to evaluate state values based on the preceding trajectory history. 
Although the original FoA employs tree-based search, we focus on its core prompting logic to provide a direct comparison of state-evaluation quality. Since \citet{klein2025fleetagentscoordinatedproblem} only considered WebShop from AgentGym, we compare with their SMC prompting approach on this benchmark. To ensure a fair comparison, we apply SMC (FoA) at the same empirically optimized resampling step used by AMC (\autoref{fig:resample_webshop}, \autoref{fig:resample_sciworld}, and \autoref{fig:resample_textcraft} of \appref{app:resampling}). For the agent's action-generation logic, we mainly adopt a ReAct-style prompting framework \cite{yao2022react}. 
To examine whether AMC transfers beyond this setting, we also implement a ReflAct-style agent \cite{kim2025reflactworldgroundeddecisionmaking} on SciWorld. We also need to choose the number of trajectories $N$. Following the discussion in \autoref{sec:bayesian_inference} in which we note that SMC approximates the optimal policy only as $N \to \infty$, we are compelled to choose a reasonably large (but tractable) trajectory count; to this end, we use $N=15$ trajectories throughout.

\noindent \textbf{Metrics.} In most applications of our method, we can use only a single final trajectory (e.g., in online shopping or deciding on a live chat response), and we will also have access to a (real or approximated) reward function $r$. The most practical application of our method is then to choose the highest reward trajectory out of the $N$ produced. To evaluate this setting, we report performance scores of the highest-reward trajectory per method using the cumulative environmental rewards provided by each benchmark. 
We conduct three independent trials using different seeds to calculate the mean scores and corresponding standard errors.
We also compare to the average scores of the top 5 trajectories within an $N$-trajectory pool in \appref{app:eval_average}, with AMC uniformly improving average-trajectory performance across tasks. We provide a mathematical proof for this result in \appref{app:amc_average}.

\begin{table}
\centering
\caption{WebShop performance, obtained with a Llama-3.2-11B-based value model for AMC.}
\scalebox{0.8}{
\begin{tblr}{
  column{3} = {c}, 
  cell{2,6,9}{2,3} = {bg=gray!10},   
  cell{3-4,7,10}{2,3} = {bg=orange!10}, 
  cell{5,8,11}{2,3} = {bg=green!10},  
  cell{2}{1} = {r=4}{l}, 
  cell{6}{1} = {r=3}{l}, 
  cell{9}{1} = {r=3}{l}, 
  vlines,
  hline{1-2} = {-}{},             
  hline{6,9,12} = {-}{},   
  row{1-Z} = {rowsep=0pt} 
}
\textbf{Policy Model} & \textbf{Method} & \textbf{Score} \\
Llama-3.2-11B & ReAct & 0.159 ($\pm$ 0.030)  \\
 & Best-of-15 (ReAct)& 0.562 ($\pm$ 0.012)  \\
 & SMC (FoA-ReAct)& 0.580 ($\pm$ 0.016)  \\
 & AMC (ReAct) & \textbf{0.625 ($\pm$ 0.009)} \\
GPT-4.1-mini & ReAct & 0.113 ($\pm$ 0.026)   \\
 & Best-of-15 (ReAct)& 0.403 ($\pm$ 0.016)  \\
 & AMC (ReAct)& \textbf{0.488 ($\pm$ 0.020)}  \\
GPT-5.1 & ReAct & 0.171 ($\pm$ 0.017) \\
 & Best-of-15 (ReAct) & 0.519 ($\pm$ 0.013)  \\
 & AMC (ReAct)&  \textbf{0.543 ($\pm$ 0.009)}
\end{tblr}
}
\label{tab:webshop_result}
\vspace{-8pt}
\end{table}

\subsection{Experimental Results}
\noindent \textbf{WebShop.} We generated $N=15$ trajectories for all methods, with the exception of the ReAct baseline, which consists of a single trajectory ($N=1$). Both AMC and SMC (FOA) used the same Llama-3.2-11B as the base value function for a fair comparison. \autoref{tab:webshop_result} presents the overall performance on WebShop. Scaling the number of sampled trajectories with Best-of-15 consistently yields substantial performance gains over the single-trajectory ReAct baseline. For the Llama-3.2-11B policy, SMC (FoA) provides only marginal improvements over Best-of-15, while AMC clearly outperforms all methods, demonstrating the benefits of learning a value function using AMC compared to the hand-crafted prompt-based value function from SMC (FoA). We also find it noteworthy that our relatively small (11B) value function can successfully guide a large-scale foundational agent like GPT-5.1, improving performance over this base policy without ever updating the underlying model.

\noindent \textbf{SciWorld.} In SciWorld, we use Llama-3.1-8B as the base model for our learned value functions in order to match the setup of \citet{agentgym-rl}, who also report results on this dataset. \autoref{tab:sciworld_result} presents the overall performance results on the SciWorld benchmark. Similar to our WebShop findings, expanding the trajectory pool with Best-of-$N$ significantly outperforms single-trajectory baselines, while AMC delivers additional consistent gains across all evaluated policies. For this benchmark, we considered both ReAct- and ReflAct-style agents; we find that ReflAct demonstrates superior performance. Our AMC method with a ReflAct-style policy achieved further improvements, highlighting the potential for compounding gains when combining AMC with more advanced decision-making agents. The consistent performance gains observed across all model scales underscore the robustness and scalability of the AMC framework.

\begin{figure*}[htbp]
    \centering
    \includegraphics[width=0.8\textwidth]{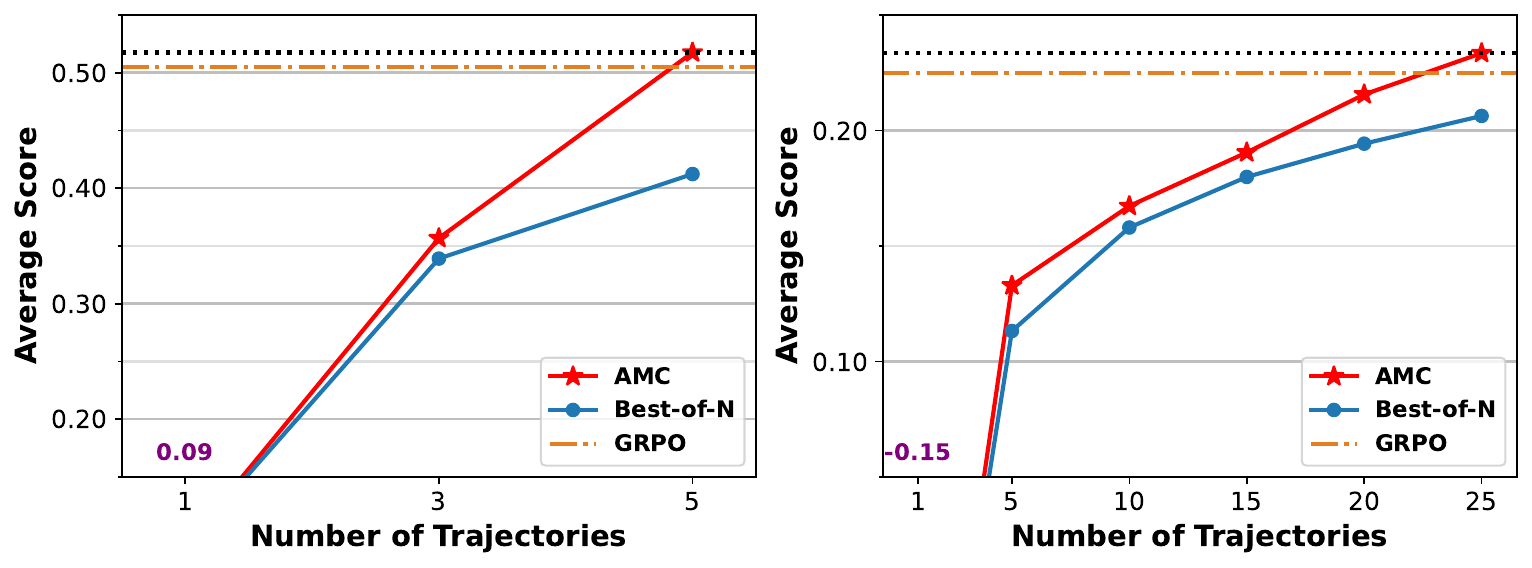}
    \caption{Comparisons to GRPO on SciWorld. Left: AMC and Best-of-$N$ with a GPT-5.1 policy and a Qwen-2.5-7B-based value function, compared to GRPO with a Qwen-2.5-7B backbone (highlighting the advantage of using AMC with black-box models). Right: AMC and Best-of-$N$ with a Qwen-2.5-3B policy and value function, compared to GRPO with the same backbone.}
    \label{fig:rl_comparison}
\end{figure*}

\begin{table}
\centering
\caption{SciWorld performance, obtained with a Llama-3.1-8B-based value model for AMC.}
\scalebox{0.8}{
\begin{tblr}{
  column{3} = {c}, 
  cell{2,5,8,11}{2,3} = {bg=gray!10},   
  cell{3,6,9,12}{2,3} = {bg=orange!10}, 
  cell{4,7,10,13}{2,3} = {bg=green!10}, 
  cell{2}{1} = {r=6}{l}, 
  cell{8}{1} = {r=3}{l}, 
  cell{11}{1} = {r=3}{l}, 
  vlines,
  hline{1-2,5,8,11,14} = {-}{},
  row{1-Z} = {rowsep=0pt}
}
\textbf{Policy Model} & \textbf{Method} & \textbf{Score} \\
Llama-3.1-8B & ReAct & 0.013 ($\pm$ 0.014) \\
 & Best-of-15 (ReAct) & 0.311 ($\pm$ 0.014) \\
 & AMC (ReAct) & \textbf{0.347 ($\pm$ 0.015)}  \\
 & ReflAct & 0.051 ($\pm$ 0.015) \\ 
 & Best-of-15 (ReflAct) & 0.347 ($\pm$ 0.010) \\
 & AMC (ReflAct) & \textbf{0.376 ($\pm$ 0.016)}  \\
GPT-4.1-mini & ReAct & 0.250 ($\pm$ 0.023)  \\
 & Best-of-15 (ReAct) & 0.616 ($\pm$ 0.020)  \\
 & AMC (ReAct)& \textbf{0.673 ($\pm$ 0.009)} \\
GPT-5.1 & ReAct & 0.090 ($\pm$ 0.049)  \\
 & Best-of-15 (ReAct) & 0.533 ($\pm$ 0.026)  \\
 & AMC (ReAct)& \textbf{0.597 ($\pm$ 0.023)}  
\end{tblr}
}
\label{tab:sciworld_result}
\vspace{-8pt}
\end{table}

\noindent \textbf{TextCraft.} Lastly, in TextCraft, we generated $N=15$ trajectories and trained a Llama-3.2-11B value function for AMC. \autoref{tab:textcraft_result} again demonstrates that multiple-trajectory approaches achieve a higher score compared to single-trajectory baselines. AMC also obtained substantial gains over Best-of-15 for smaller policy models. However, we observe performance saturation as the base policy improves. This is evident with GPT-5.1, where Best-of-15 even exceeds AMC. We hypothesize that this performance inversion is due to the strength of GPT-5.1 as a prior. For TextCraft, GPT-5.1 generates shorter, higher-confidence trajectories, leaving little room for improvement and reducing the diversity of the value function's training data. This makes it more difficult for AMC to separate out promising trajectories, leading it to accidentally prune some. Based on this finding, we expect AMC to be most useful for model-task pairs where the model produces good but not uniformly perfect trajectories. Crucially, these findings do not imply that AMC loses utility as models get better because $(i)$ there is no foreseeable upper bound on real-world problem complexity and $(ii)$ cost constraints will favour smaller, more efficient models that may not be performance-maximal. 

\begin{table}
\centering
\caption{TextCraft performance, obtained with a Llama-3.2-11B-based value model for AMC.}
\scalebox{0.8}{
\begin{tblr}{
  column{3} = {c}, 
  cell{2,5,8}{2,3} = {bg=gray!10},   
  cell{3,6,9}{2,3} = {bg=orange!10}, 
  cell{4,7,10}{2,3} = {bg=green!10},  
  cell{2}{1} = {r=3}{l}, 
  cell{5}{1} = {r=3}{l}, 
  cell{8}{1} = {r=3}{l}, 
  vlines,
  hline{1-2} = {-}{}, 
  hline{5,8,11} = {-}{},
  row{1-Z} = {rowsep=0pt}
}
\textbf{Policy Model} & \textbf{Method} & \textbf{Score} \\
Llama-3.2-11B & ReAct & 0.102 ($\pm$ 0.029) \\
 & Best-of-15 (ReAct) & 0.296 ($\pm$ 0.019) \\
 & AMC (ReAct) & \textbf{0.543 ($\pm$ 0.057)} \\
GPT-4.1-mini & ReAct & 0.432 ($\pm$ 0.055) \\
 & Best-of-15 (ReAct) & 0.728 ($\pm$ 0.010) \\
 & AMC (ReAct) & \textbf{0.852 ($\pm$ 0.020)}\\
GPT-5.1 & ReAct & 0.691 ($\pm$ 0.012) \\
 & Best-of-15 (ReAct) & \textbf{0.889 ($\pm$ 0.000)} \\
 & AMC (ReAct) & 0.790 ($\pm$ 0.021)  \\
\end{tblr}
}
\label{tab:textcraft_result}
\vspace{-8pt}
\end{table}

\subsection{Comparison with GRPO}

GRPO \cite{shao2024deepseekmathpushinglimitsmathematical} is a KL-regularized RL method that has become common for LLM-based policies. As with AMC, GRPO attempts to model $\policyopt$ from \autoref{eq:variational_inference}, but unlike AMC, it requires full access to the policy's underlying parameters to perform weight updates. 
GRPO also requires online trajectory rollouts during training, which incurs additional computational overhead, over the simple offline regression task required to learn $\valuetrain$ (\autoref{sec:learning_value}).
To clarify, AMC is not intended to replace GRPO, but to serve as a viable alternative when GRPO is not possible. We thus consider GRPO to be an oracle rather than a typical baseline.

In this section, we put AMC and GRPO head-to-head to better understand whether our approach is a viable alternative to policy gradient RL. For a direct comparison to \citet{agentgym-rl}, we use the same LLM policy backbone (Qwen-2.5-3B) for both AMC and GRPO; we also use this same base model to train our value function, thus controlling for base model and pre-training biases. Additionally, to highlight AMC's compatibility with black-box policies, we consider an alternative comparison setup: using a frontier LLM agent (GPT-5.1) as the prior policy for AMC and training a value function backed by the same LLM as the one GRPO fine-tunes for its policy (Qwen-2.5-7B).

\autoref{fig:rl_comparison} illustrates these comparisons on the SciWorld benchmark. Note that all GRPO scores are taken from \citet{agentgym-rl}. When backed by a GPT-5.1 policy (\autoref{fig:rl_comparison}, left), AMC outperforms GRPO with only $N=5$ trajectories, achieving a performance level that remains unattainable for Best-of-$N$ given the same trajectory budget. Furthermore, when using the same LLM backbone for the policy of both GRPO and AMC (\autoref{fig:rl_comparison}, right), we find that when scaling AMC to 25 trajectories it outperforms fully fine-tuning this policy with GRPO. These results highlight that $(i)$ AMC performs comparably to GRPO using the same prior policy given enough trajectories (but requiring fewer than Best-of-$N$), and $(ii)$ it can benefit even more from better black-box priors. In \appref{app:grpo_cost}, we fine-tuned our own GRPO baselines to provide a comprehensive performance and cost analysis. To underscore the relative cost requirements of AMC and GRPO, we note that all AMC experiments were performed on a workstation with two RTX 6000 Ada desktop GPUs, whereas GRPO experiments required a node with eight A100 GPUs.

\subsection{Impact of Value Model Optimization}

\noindent \textbf{Training vs. Prompting.} Since our trained value functions are initialized from pre-trained (instruction-tuned) LLMs, we ask the question: is training necessary? Can we achieve the same performance by prompting alone? We thus compare AMC to a (heuristic) prompt-based approach, SMC (Zero-shot), where we simply ask the pre-trained LLM to estimate the value of a given state (see \autoref{app:templates} for all prompts). \autoref{tab:critic_training} illustrates this comparison across both open-weight and black-box policies, revealing consistent performance gains from training. Notably, SMC (Zero-Shot) delivers inconsistent performance gain over the Best-of-$N$ baseline, underscoring that raw pre-trained knowledge is insufficient for precise state-value estimation in these agentic environments.

\begin{table}[h] 
\centering
\caption{Impact of value function optimization on performance.}
\scalebox{0.775}{
\begin{tblr}{
  colspec = {l l l c},
  vlines,
  hline{1,2,5,8,11,14,17,20,23,26,29} = {-}{}, 
  cell{2,11,20}{1} = {r=9}{l, m}, 
  cell{2,5,8,11,14,17,20,23,26}{2} = {r=3}{l, m}, 
  cell{2-3,5-6,8-9,11-12,14-15,17-18,20-21,23-24,26-27}{3-5} = {bg=orange!10}, 
  cell{4,7,10,13,16,19,22,25,28}{3-4} = {bg=green!10},  
  row{1} = {font=\bfseries},
  row{1-Z} = {rowsep=0pt}  
}
Dataset & Policy Model & Method & Score \\
Webshop & Llama-3.2-11B & Best-of-15 & 0.562 ($\pm$ 0.012) \\
        &               & SMC (Zero-shot) & 0.556 ($\pm$ 0.014) \\
        &               & AMC & \textbf{0.625 ($\pm$ 0.009)} \\
        & GPT-4.1-mini  & Best-of-15 & 0.403 ($\pm$ 0.016) \\
        &               & SMC (Zero-shot) & 0.418 ($\pm$ 0.014)\\
        &               & AMC & \textbf{0.488 ($\pm$ 0.020)} \\
        & GPT-5.1  & Best-of-15 &  0.519 ($\pm$ 0.013)\\
        &               & SMC (Zero-shot) & 0.522 ($\pm$ 0.008) \\
        &               & AMC & \textbf{0.543 ($\pm$ 0.009)} \\        
SciWorld & Llama-3.1-8B & Best-of-15 & 0.311 ($\pm$ 0.014) \\
         &               & SMC (Zero-shot) & 0.305 ($\pm$ 0.013) \\
         &               & AMC & \textbf{0.347 ($\pm$ 0.015)} \\
         & GPT-4.1-mini  & Best-of-15 & 0.616 ($\pm$ 0.020) \\
         &               & SMC (Zero-shot) & 0.640 ($\pm$ 0.022) \\
         &               & AMC & \textbf{0.673 ($\pm$ 0.009)} \\
         & GPT-5.1  & Best-of-15 & 0.533 ($\pm$ 0.026)\\
         &               & SMC (Zero-shot) & 0.506 ($\pm$ 0.041) \\
         &               & AMC & \textbf{0.597 ($\pm$ 0.023)} \\         
TextCraft & Llama-3.2-11B & Best-of-15 & 0.296 ($\pm$ 0.019) \\
          &               & SMC (Zero-shot) & 0.296 ($\pm$ 0.043) \\
          &               & AMC & \textbf{0.543 ($\pm$ 0.057)} \\
          & GPT-4.1-mini  & Best-of-15 & 0.728 ($\pm$ 0.010) \\
          &               & SMC (Zero-shot) & 0.765 ($\pm$ 0.025) \\
          &               & AMC & \textbf{0.852 ($\pm$ 0.020)} \\
          & GPT-5.1  & Best-of-15 & \textbf{0.889 ($\pm$ 0.000)}\\
          &               & SMC (Zero-shot) & 0.815 ($\pm$ 0.026) \\
          &               & AMC & 0.790 ($\pm$ 0.021) \\          
\end{tblr}}
\label{tab:critic_training}
\vspace{-8pt}
\end{table}

\begin{figure*}[t]
    \centering
    \includegraphics[width=0.33\linewidth]{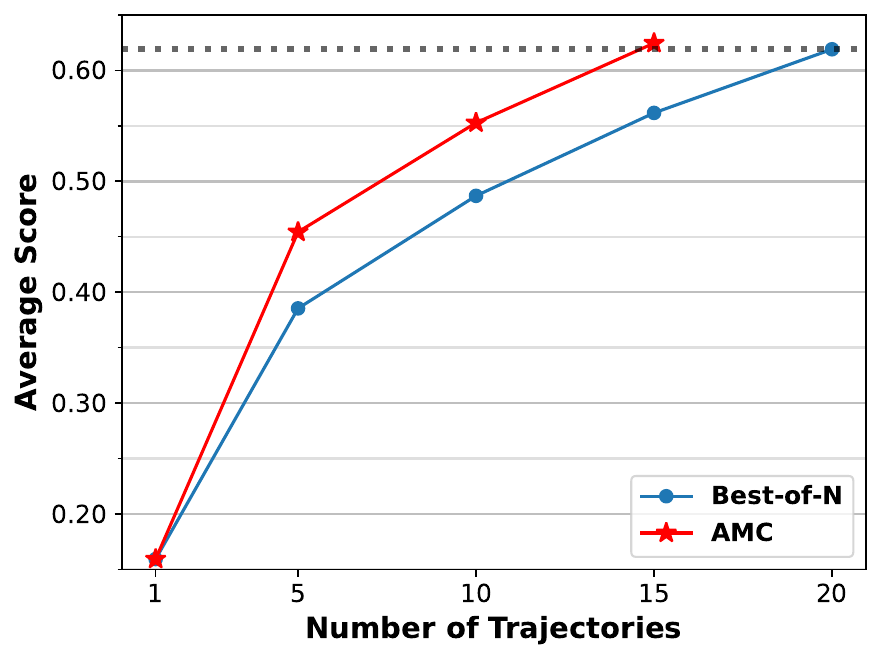}
    \includegraphics[width=0.33\linewidth]{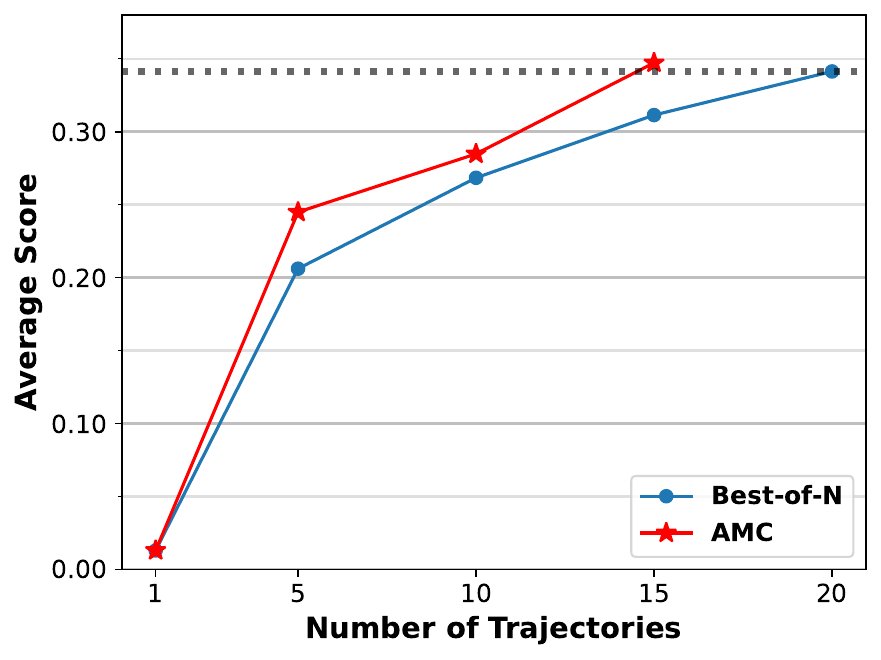}
    \includegraphics[width=0.33\linewidth]{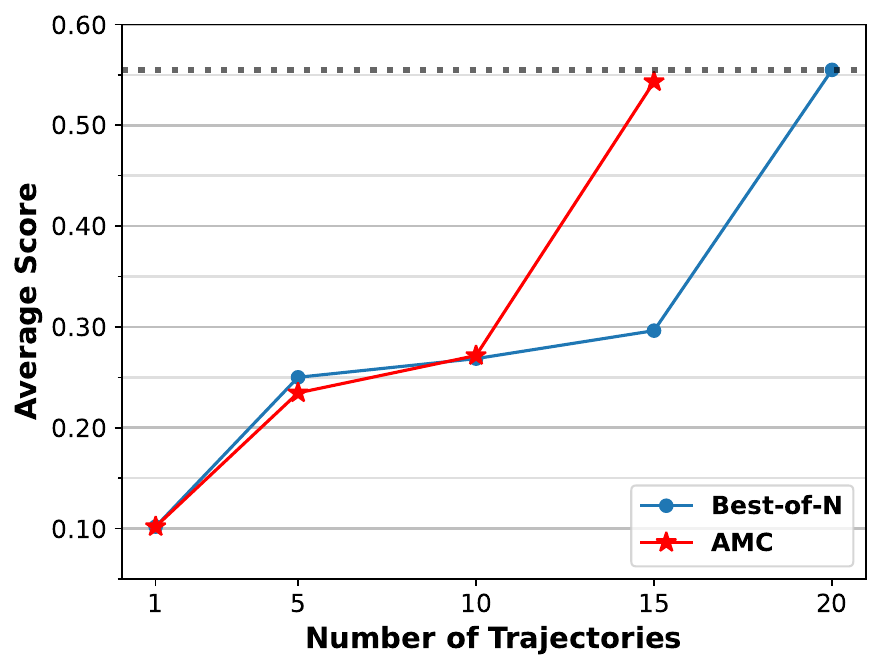}
    \caption{Best-of-$N$ vs. AMC across tasks. Left: WebShop with Llama-3.2-11B policy, value function. Middle: SciWorld with Llama-3.1-8B policy, value function. Right: TextCraft with Llama-3.2-11B policy, value function.}
    \label{fig:num_particles}
\end{figure*}

\subsection{Cost-Efficiency and Scaling}
\noindent \textbf{Trajectory Counts.} The effectiveness of our approach relies on sampling multiple trajectories at test-time, which naturally raises questions around inference cost and scalability. To evaluate the practical feasibility of AMC, we compare its efficiency against Best-of-$N$ in \autoref{fig:num_particles}. 
When the number of trajectories is limited to 5 or 10, AMC demonstrates performance gains over the baseline with the exception of TextCraft. This exception is likely due to the reduced diversity of visited states arising from TextCraft's lower relative task complexity, and is further hampered by AMC's resampling; increasing the trajectory count resolves this by promoting more state space exploration. As $N$ increases, AMC generally achieves performance parity with Best-of-20 using only $N=15$ trajectories, thereby reducing test-time-compute by approximately 25\%. Note that the value function is only evaluated on the initial state and on resampling steps and requires only a single prefill, so its computational overhead is small compared to the cost of the LLM policy, which performs both prefill and decoding at every step.

\begin{table}[h] 
\centering
\caption{Performance and cost of a lightweight policy using AMC versus the more expensive GPT-5.1.}
\scalebox{0.73}{
\begin{tblr}{
  cell{2,4,6}{3-5} = {bg=orange!10}, 
  cell{3,5,7}{3-5} = {bg=green!10},  
  colspec = {l l l c c},
  vlines,
  hline{1,2,4,6,8} = {-}{},
  cell{2}{1} = {r=2}{l}, 
  cell{4}{1} = {r=2}{l},
  cell{6}{1} = {r=2}{l}, 
    row{1-Z} = {rowsep=0pt}
}
\textbf{Dataset} & \textbf{Policy Model} & \textbf{Method} & \textbf{Score} & \textbf{Cost (USD)} \\
WebShop   & GPT-5.1       & Best-of-15 & \textbf{0.519 ($\pm$ 0.013)}  & 0.39 \\
          & GPT-4.1-mini & AMC      & 0.488 ($\pm$ 0.020)  & \textbf{0.14} \\
SciWorld  & GPT-5.1       & Best-of-15 & 0.533 ($\pm$ 0.026) & 0.18 \\
          & GPT-4.1-mini & AMC      &   \textbf{0.673 ($\pm$ 0.009)}     & \textbf{0.06}\\
TextCraft & GPT-5.1       & Best-of-15 & \textbf{0.889 ($\pm$ 0.000)}   & 0.45 \\
          & GPT-4.1-mini & AMC      & 0.852 ($\pm$ 0.020) & \textbf{0.21}\\
\end{tblr}
}
\label{tab:gpt4.1_vs_gpt5.1}
\vspace{-8pt}
\end{table}

\noindent \textbf{Policy Cost.} Expanding on these results, we find that our method also enables smaller black-box models to achieve performance comparable to their larger counterparts, effectively narrowing the gap between lightweight and high-capacity policies. We empirically substantiate this finding in \autoref{tab:gpt4.1_vs_gpt5.1}, where costs are calculated based on the average input and output token counts across tasks.\footnote{\href{https://platform.openai.com/docs/pricing?latest-pricing=standard}{OpenAI pricing page.}} GPT-4.1-mini paired with AMC maintains a performance level similar to GPT-5.1 while reducing total cost by at least 50\%. 

\noindent \textbf{Prior Surrogates.} Part of AMC's upfront cost is the initial collection of trajectories from the prior LLM policy $\pi$ to train the value model. When $\pi$ is behind an API, this upfront data acquisition can become expensive. To circumvent this, we explore using small, open-weight surrogate policies to sample training trajectories instead of using $\pi$ directly; lightweight models can generally be self-hosted on existing GPUs to circumvent API costs. \autoref{tab:transferability} shows results; using a surrogate policy uniformly outperforms Best-of-$N$. Although sampling trajectories from $\pi$ itself performs better than using a surrogate policy, these results still demonstrate transferability of $\valuetrain$ as a viable way to reduce cost.

\begin{table}[t]
\centering
\caption{Transferability of value functions where base models are Llama-3.2-11B (WebShop, TextCraft) and Llama-3.1-8B (SciWorld).}
\scalebox{0.715}{
\begin{tblr}{
  colspec = {l l l l c},
  vlines,
  cell{2,5,8}{3,4,5} = {bg=orange!10}, 
  cell{3,4,6,7,9,10}{3,4,5} = {bg=green!10},  
  hline{1,2,5,8,11} = {-}{solid}, 
  cell{2,5,8}{1} = {r=3}{l},
  cell{2,5,8}{2} = {r=3}{l},
  cell{3,6,9}{3} = {r=2}{l}, 
  row{1-Z} = {rowsep=0pt}
}
\textbf{Dataset} & \textbf{Policy Model} & \textbf{Method} & \textbf{Surrogate} & \textbf{Score} \\
WebShop       & GPT-4.1-mini  & Best-of-15 & - & 0.403 ($\pm$ 0.016) \\
              &               &  AMC & Llama-3.2-11B & \textbf{0.467 ($\pm$ 0.028)} \\
              &               &       & None  & \textbf{0.488 ($\pm$ 0.020)} \\
SciWorld      & GPT-4.1-mini  & Best-of-15 & - & 0.616 ($\pm$ 0.020)\\
              &               & AMC & Llama-3.1-8B  & \textbf{0.645 ($\pm$ 0.023)} \\
              &               &      & None  & \textbf{0.673 ($\pm$ 0.009)} \\
TextCraft    & GPT-4.1-mini  & Best-of-15 & - & 0.728 ($\pm$ 0.010) \\
              &               & AMC & Llama-3.2-11B & \textbf{0.790 ($\pm$ 0.025})  \\
              &               &      & None   & \textbf{0.852 ($\pm$ 0.020)} \\
\end{tblr}
}
\label{tab:transferability}
\vspace{-8pt}
\end{table}


\section{Conclusion, Limitations, and Future Work}
\label{sec:conclusion}

Most of today's LLM-based interactions are hidden behind black-box APIs, and yet most agentic applications (e.g., e-commerce, software engineering) are specialized tasks that could benefit from task-specific tuning. Is it possible to emulate the effects of RL-based fine-tuning with black-box policies? Our work answers this question in the affirmative. Following the well-known bridge between KL-regularized RL and Bayesian inference, we approximate the agent's optimal policy without directly training it through RL, but rather by Monte Carlo sampling from its Bayesian posterior distribution. Our algorithm, Agentic Monte Carlo (AMC), implements sequential importance resampling using a learned value function to guide samples from the black-box LLM prior toward samples from the intractable optimal posterior. 

We empirically showed the effectiveness of our approach on the WebShop, SciWorld, and TextCraft benchmarks, highlighting improvements over existing baselines.
Notably, AMC is able to improve the performance of smaller black-box models to match that of larger black-box models, thus incurring significantly cheaper API costs.
We further showcased our method against policy gradient RL methods such as GRPO, demonstrating that AMC can achieve equivalent performance as we scale test-time compute.

There are a number of avenues for future work on AMC. For one, our learned value function will never perfectly approximate the true value function. Potential ways to improve the training procedure include temporal difference learning, \cite{sutton1988learning}, reward shaping \cite{ng1999policy}, or better approximating \autoref{eq:loss_func_ideal}. Another area for improvement is the computational overhead that comes with having to sample multiple trajectories in parallel during inference. Greater efficiency may be attained, for example, by distilling AMC or some component of AMC into a smaller network \cite{hinton2015distilling}. AMC can also be extended to broader settings, like multi-agent black-box agents. As more frontier models are released with proprietary, opaque architectures, AMC provides a principled mechanism to orient and optimize these black-box agents for task-specific execution. 
Additionally, several promising directions exist for enhancing resampling strategies: developing auxiliary models trained to trigger resampling, leveraging LLM-based judges to identify specific trajectory states where resampling is likely to yield the highest marginal gain \cite{feng2026conformal}, and employing adaptive resampling frameworks or hyperparameter optimization techniques to arrive at more sophisticated choices of resampling intervals \citep{Doucet2001,chopin2020introduction}. Lastly, there remains theory to be built around our method. AMC introduces potential error in the reward model \citep{huang2025best} or value function, and it would be useful to understand its interplay with the already-extensive body of work around SMC \citep{crisan2000convergence,moral2004feynman,marion2023finite}.

\section*{Impact Statement}

This paper presents work whose goal is to advance the field of Machine
Learning, specifically the use of LLM agents. Our work focuses on providing statistically principled sampling methods to extend reinforcement learning techniques to black-box LLMs. This work is unlikely to cause broader societal impact outside of the general impacts of agentic AI.


\bibliography{bib}
\bibliographystyle{icml2026}

\newpage
\appendix
\onecolumn

\section{Theory}
\label{app:theory}
\subsection{Sequential Importance Resampling}\label{app:sir_alg}
In this section, we record the full SIR algorithm, \autoref{alg:sir} \citep{Doucet2001}, with our weight update rule including the (soft) value function.

\begin{algorithm}[h]
   \caption{Sequential Importance Resampling}
   \label{alg:sir}
\begin{algorithmic}[1]
   \REQUIRE Reward function $\reward$, value function $\valuefunc$, prior policy $\policy$, resampling criterion $\resamplingcrit$, trajectory count $N$, time horizon $T$
   \ENSURE Posterior trajectories $\traj^{(i)}$ and weights $w_T^{(i)}$ for $i=1,\dots,N$
   
   \STATE \textbf{Initialization:} For $i = 1, \dots, N$, sample $\state[0]^{(i)} \sim \policy(\state[0])$ and set $w_0^{(i)} = 1/N$.
   
   \FOR{$t = 1$ \textbf{to} $T$}
      \FOR{$i = 1$ \textbf{to} $N$}
         \STATE Sample $\tstate^{(i)} \sim \policy(\state \mid \state[t-1]^{(i)})$
         \STATE Set $\ttraj[0:t]^{(i)} \leftarrow \left(\trajcurr[t-1]^{(i)}, \tstate^{(i)}\right)$
         \STATE Compute importance weights:
         \vspace{-2pt}
         \begin{equation*}
             w_t^{(i)} = w_{t-1}^{(i)} e^{\valuefunc(\tstate^{(i)}) - \valuefunc(\state[t-1]^{(i)}) + \reward(\state[t-1]^{(i)})}
                  \vspace{-2pt}
                  \end{equation*}
      \ENDFOR
            
      \IF{$\resamplingcrit(t, \{w_t^{(i)}\}_{i=1}^N)$}
          \STATE Resample $N$ choices of trajectory $\{\trajcurr^{(i)}\}_{i=1}^N$ from $\{\ttraj[0:t]^{(i)}\}_{i=1}^N$ according to the normalized weights: \texttt{Multinomial}$\left(N, \left\{\frac{w_t^{(i)}}{\sum_{j=1}^N w_t^{(j)}}\right\}_{i=1}^N\right)$.
          \STATE Reset weights $w_t^{(i)} \leftarrow 1/N$.
      \ELSE
          \STATE $\trajcurr^{(i)} \leftarrow \ttraj[0:t]^{(i)}$ for $i=1,\dots,N$.
      \ENDIF
   \ENDFOR
\end{algorithmic}
\end{algorithm}

\subsection{Equivalence between KL-Regularized RL and Bayesian Inference}\label{app:theory_equivalence}

For completeness, we include the proof of the equivalence between KL-regularized RL and Bayesian inference discussed by \citet{korbak2022rl}. This is also a known result in the control-as-inference literature \citep{levine2018reinforcement}, and mathematically the same as the fundamental proof justifying variational inference (VI) as a method.
\begin{proposition}
If we let
\begin{equation}
    \policyopt^\mathrm{(Bayes)}(\traj) = \frac{1}{Z} \policy(\traj) e^{\reward(\traj)/\beta},
\end{equation}
where $Z = \int \policy(\traj) e^{\reward(\traj)/\beta} d\traj$ is a normalizing constant, and let 
\begin{equation}
    \policyopt^\mathrm{(VI)} = \argmax_{\policytrain} \mathcal{J}(\policytrain) = \argmax_{\policytrain} \left( \E_{\policytrain(\traj)} \left[\reward(\traj)\right] - \beta \KL\left[\policytrain \Vert \policy\right] \right),
\end{equation}
then $\policyopt^\mathrm{(Bayes)} = \policyopt^\mathrm{(VI)}$.
\end{proposition}

\begin{proof}
Expanding $\mathcal{J}(\policytrain)$, we can rewrite it in terms of $\policy_*^\mathrm{(Bayes)}$:

\begin{align*}
\mathcal{J}(\policytrain) &= \E_{\policytrain(\traj)} \left[\reward(\traj)\right] - \beta \KL\left[\policytrain \Vert \policy\right] \\
    &= \int \policytrain(\traj) \reward(\traj) d\traj - \beta \int \policytrain(\traj) \log \frac{\policytrain(\traj)}{\policy(\traj)} d\traj \\
    &= \beta \int \policytrain(\traj) \left[ \frac{\reward(\traj)}{\beta} - \log \frac{\policytrain(\traj)}{\policy(\traj)} \right] d\traj \\
    &= \beta \int \policytrain(\traj) \log \frac{\policy(\traj) e^{\reward(\traj)/\beta}}{\policytrain(\traj)} d\traj \\
    &= \beta \int \policytrain(\traj) \log \frac{Z \policy_*^\mathrm{(Bayes)}(\traj)}{\policytrain(\traj)} d\traj \\
    &= \beta \log Z - \beta \int \policytrain(\traj) \log \frac{\policytrain(\traj)}{\policy_*^\mathrm{(Bayes)}(\traj)} d\traj \\
    &= \beta \log Z - \beta \KL\left[\policytrain \Vert \policy_*^\mathrm{(Bayes)}\right]
\end{align*}

Because $\beta \log Z$ is a constant independent of $\policytrain$, the argmax of $\mathcal{J}(\policytrain)$ is equivalent to the argmin of $\KL\left[\policytrain \Vert \policy_*^\mathrm{(Bayes)}\right]$. Since the KL divergence is minimized uniquely at zero when its arguments are equal, it follows that $\policy_*^\mathrm{(VI)} = \policy_*^\mathrm{(Bayes)}$.
\end{proof}

\subsection{Recursive Decomposition of Importance Weights}\label{app:theory_decomp}

Here we show how the importance weights can be recursively decomposed. This proof is similar, for example, to that of \citet[Appendix A.4]{piche2018probabilistic}.

\begin{proposition}
The importance weight $w_t = \frac{\pi_*(\trajcurr)}{\policy(\trajcurr)}$ (\autoref{eq:importance_weight_def}) obeys the recursive decomposition from \autoref{eq:importance_weight_recursive}:
\begin{equation*}
    w_t = w_{t-1} \cdot e^{\valuefunc(\state) - \valuefunc(\state[t-1]) + \frac{\reward(\state[t-1])}{\beta}},
\end{equation*}
where $\valuefunc(\state) = \log\E_{\policy(\trajfuture[t+1] \mid \state)}[e^{\frac{1}{\beta}\sum_{\tau=t}^T \reward(\state[\tau])}]$ is the soft value function.
\end{proposition}

\begin{proof}
First, we expand the marginal posterior $\pi_*(\trajcurr)$ using the definition of $\policyopt$ from \autoref{eq:bayesian_inference}:
\begin{align*}
\pi_*(\trajcurr) &= \int \pi_*(\traj) d\trajfuture[t+1] \\
    &= \frac{1}{Z} \int \policy(\traj) e^{\frac{1}{\beta} \sum_{\tau=0}^{T} \reward(\state[\tau])} d\trajfuture[t+1] \\
    &= \frac{1}{Z} \policy(\trajcurr) e^{\frac{1}{\beta} \sum_{\tau=0}^{t-1} \reward(\state[\tau])} \int \policy(\trajfuture[t+1] \mid \state) e^{\frac{1}{\beta} \sum_{\tau=t}^{T} \reward(\state[\tau])} d\trajfuture[t+1] \\
    &= \frac{1}{Z} \policy(\trajcurr) e^{\frac{1}{\beta} \sum_{\tau=0}^{t-1} \reward(\state[\tau])} e^{\valuefunc(\state)}.
\end{align*}
We then use this expression to expand out the importance weight:
\begin{align*}
w_t &= \frac{\pi_*(\trajcurr)}{\policy(\trajcurr)} \\
    &= \frac{1}{Z} e^{\frac{1}{\beta} \sum_{\tau=0}^{t-1} \reward(\state[\tau])} e^{\valuefunc(\state)}.
\end{align*}
From this expression, we can isolate $w_{t-1}$ to get a recursive decomposition:
\begin{align*}
w_t &= \left( \frac{1}{Z} e^{\frac{1}{\beta} \sum_{\tau=0}^{t-2} \reward(\state[\tau])} e^{\valuefunc(\state[t-1])} \right) \cdot \frac{e^{\frac{1}{\beta} \reward(\state[t-1])} e^{\valuefunc(\state)}}{e^{\valuefunc(\state[t-1])}} \\
    &= w_{t-1} \cdot e^{\valuefunc(\state) - \valuefunc(\state[t-1]) + \frac{\reward(\state[t-1])}{\beta}}.
\end{align*}
\end{proof}

\subsection{SMC Improves the Average-Case Trajectory}\label{app:amc_average}
Here, we show in a simplified setting that SMC will always improve the expected average reward across $N$ trajectories. In a setting with sparse, binary rewards, we can express the usefulness of a state $s_t^{(i)}$ by its success probability, $p_t^{(i)} = P(r(s_T^{(i)}) = 1 \mid s_t^{(i)})$.

\begin{itemize}
    \item Given $N$ different states $\{s_t^{(1)}, \ldots, s_t^{(N)}\}$ and any step $t$, the expected average reward over the $N$ trajectories in a single Best-of-N run is:
    \[
    \mathbb{E}[\mathrm{RewardN}] = \frac{1}{N} \sum_{i=1}^N p_t^{(i)}
    \]
    
    \item On the other hand, with AMC (assuming a perfectly learned value function) let us assume resampling occurs once at timestep $t$. Here, we have a simplified value function and weight, so that:
    \begin{itemize}
        \item $V(s_t^{(i)}) = \log[e p_t^{(i)} + (1 - p_t^{(i)})]$. This comes from the definition of $V(\cdot)$ (under \autoref{eq:importance_weight_recursive}), setting $r(s_t) = 0$ for $t \neq T$, and from expanding out the expectation in the definition with binary rewards. (For this proof we set $\beta=1$.)
        \item The unnormalized weights are:
        \[
        w_t^{(i)} = e^{V(s_t^{(i)}) - V(s_0^{(i)})} = \frac{e p_t^{(i)} + 1 - p_t^{(i)}}{e p_0^{(i)} + 1 - p_0^{(i)}} = \frac{(e - 1)p_t^{(i)} + 1}{(e - 1)p_0^{(i)} + 1}
        \]
        \item We also let $\bar{w}_t^{(i)}$ for $i=1, \ldots, N$ be the renormalized weights used for sampling at timestep $t$. Since all trajectories share the same initial state, the denominator cancels out during normalization, yielding the following value:
        \[
        \bar{w}_t^{(i)} = \frac{(e - 1)p_t^{(i)} + 1}{(e - 1)\sum_{j=1}^N p_t^{(j)} + N}
        \]
    \end{itemize}
\end{itemize}

The average resulting reward will be:
\begin{align*}
\mathbb{E}[\mathrm{RewardAMC}] &= \frac{1}{N} \sum_{k=1}^N \mathbb{E}_{\tilde{s}_t^{(k)} \sim \bar{w}_t} \left[ \tilde{p}_t^{(k)} \right] \\
&= \frac{1}{N} \sum_{k=1}^N \left( \sum_{i=1}^N \bar{w}_t^{(i)} p_t^{(i)} \right) \\
&= \sum_{i=1}^N \bar{w}_t^{(i)} p_t^{(i)} \\
&= \sum_{i=1}^N \left( \frac{(e - 1)p_t^{(i)} + 1}{(e - 1)\sum_{j=1}^N p_t^{(j)} + N} \right) p_t^{(i)} \\
&= \frac{(e - 1)\sum_{i=1}^N (p_t^{(i)})^2 + \sum_{i=1}^N p_t^{(i)}}{(e - 1)\sum_{j=1}^N p_t^{(j)} + N} \\
&\ge \frac{(e - 1)\left[ \frac{1}{N}\left(\sum_{i=1}^N p_t^{(i)}\right)^2 \right] + \sum_{i=1}^N p_t^{(i)}}{(e - 1)\sum_{j=1}^N p_t^{(j)} + N} \quad \text{(Cauchy-Schwarz inequality)} \\
&= \frac{\frac{1}{N}\left(\sum_{i=1}^N p_t^{(i)}\right) \left[ (e - 1)\sum_{i=1}^N p_t^{(i)} + N \right]}{(e - 1)\sum_{j=1}^N p_t^{(j)} + N} \\
&= \frac{1}{N} \sum_{i=1}^N p_t^{(i)} \\
&= \mathbb{E}[\mathrm{RewardN}]
\end{align*}

where the \textbf{inequality is strict} unless all states have the same success probability (which is essentially never true in practice). Hence, AMC will nearly always result in higher average reward compared to the average of the $N$ trajectories in Best-of-N. We note that in this theoretical setup, AMC will not always have higher maximum reward, particularly when there already exist very strong trajectories (e.g., a two-trajectory binary success case where one trajectory is guaranteed to succeed, and one is guaranteed to fail).

\section{Experiment Setup}
\label{app:setup}
\noindent \textbf{Hardware and Computational Costs.} All AMC experiments were conducted on a workstation equipped with two NVIDIA RTX 6000 Ada Generation GPUs (48GB VRAM each) and 1TB of system RAM. Under this configuration, the training duration for the Llama-3.1-8B value model was approximately 2 hours, while the Llama-3.2-11B value model required roughly 3 hours to complete. A comprehensive specification of hyperparameters is detailed in \autoref{tab:setup}.

\noindent \textbf{Experimental Protocol.} To ensure statistical robustness and account for stochasticity in model generation, every experiment was executed across three independent trials using distinct random seeds. We report the mean scores alongside their corresponding standard errors to provide a clear measure of performance stability and variance. 

\begin{table}[h] 
\centering
\caption{Detailed implementation settings.}
\scalebox{0.85}{
\begin{tblr}{
  colspec = {Q[l,m,2.2cm] l l}, 
  vlines,
  hline{1,2,12,17} = {-}{}, 
  cell{2}{1} = {r=10}{l,m}, 
  cell{12}{1} = {r=5}{l,m}, 
  cell{15,16}{3} = {m},
}
\textbf{Category} & \textbf{Hyperparameter} & \textbf{Value} \\
Value Model Optimization & Optimizer & AdamW \\
 & Learning rate & $1 \times 10^{-5}$ \\
 & Scheduler & Cosine \\
 & Warmup ratio & $1 \times 10^{-1}$ \\
 & Max gradient norm & 1 \\
 & Loss function & Exponentially weighted MSE \\
 & Training epochs & 3 \\
 & LoRA rank ($r$) & 8 \\
 & LoRA alpha ($\alpha$) & 16 \\
 & Output format & Scalar (float) value \\
Inference Configuration & Temperature ($\beta$) & 1.0 \\
 & Top-$p$ & 0.95 \\
 & Max output length & 4096 \\
 & Max steps & {WebShop: 10 \\ SciWorld: 20 \\ TextCraft: 20} \\
 & Resampling steps & {WebShop: [6] \\ SciWorld: [4, 12] \\ TextCraft: [4]} \\
\end{tblr}
}
\label{tab:setup}
\end{table}

\noindent \textbf{Datasets.} We evaluate our framework on three distinct environments.
\textbf{WebShop} is a simulated e-commerce environment that evaluates an agent’s ability to navigate a multi-step web interface to purchase products satisfying multi-attribute constraints. It features an open-ended action space and a continuous reward based on attribute overlap with the target item. \textbf{SciWorld} evaluates long-horizon reasoning through elementary science tasks within a text-based simulator. This environment features a massive combinatorial action space of approximately 200K action-object combinations per step. Success requires strict adherence to procedural instructions, as deviation from the required sequence can trigger a terminal failure reward of -1. Performance is measured by sub-goal milestones, where agents earn incremental points for completing discrete required stages of an experiment. \textbf{TextCraft} is a Minecraft-inspired benchmark where agents craft items using hierarchical recipes. Using three primary actions (get, craft, and inventory), agents must manage compositional dependencies to earn a sparse binary reward of 1, granted only upon successful creation of the target item. While the environment features recipes across four levels of depth, we evaluated tasks at depths 3 and 4 since performance on lower-depth tasks has already reached saturation with current policy models \cite{agentgym-rl}.

\noindent \textbf{Training Set Construction.} We first sampled $\mathcal{G}$ trajectories per task and only included in our dataset the sample with the highest cumulative reward. In total, we generated $M=$1.4k trajectories for WebShop and 1.1k for SciWorld, both using $\mathcal{G}=3$. For TextCraft, given the limited availability of only 44 tasks, we set $\mathcal{G}=8$ to promote more diversity in the sampled trajectories. Following previous works \citep{agentgym, agentgym-rl}, the maximum number of steps was fixed to $T=10$ for WebShop and $T=20$ for both SciWorld and TextCraft. By treating each individual step within a trajectory as a separate data point (see \autoref{sec:learning_value}), we produced approximately $P=14$k, $22$k, and $880$ samples for WebShop, SciWorld, and TextCraft, respectively, which we then further split 80\%/20\% for our training and validation sets.

\section{Prompt Templates}
\label{app:templates}

\begin{tcolorbox}[title=ReAct Prompt Template for WebShop, colback=gray!5, colframe=black]
\small
\texttt{You are web shopping.\\
I will give you instructions about what to do.\\
You have to follow the instructions.\\
Every round I will give you an observation and a list of available actions, you have to respond an action based on the state and instruction.\\
You can use search action if search is available.\\
You can click one of the buttons in clickables.\\
An action should be of the following structure:\\
search[keywords]\\
click[value]\\
If the action is not valid, perform nothing.\\
Keywords in search are up to you, but the value in click must be a value in the list of available actions.\\
Remember that your keywords in search should be carefully designed.\\
Your response should use the following format:\\
\\
Thought:\\
I think ... \\
\\
Action: \\
click[something]}
\end{tcolorbox}

\begin{tcolorbox}[title=ReAct Prompt Template for SciWorld, colback=gray!5, colframe=black]
\small
\texttt{You are an agent for science world. Every round I will give you an observation, you have to respond an action based on the observation to finish the given task. Here are the actions you may take:\\
\\
$[${"action": "open/close OBJ", "description": "open/close a container"}, \\
{"action": "de/activate OBJ", "description": "activate/deactivate a device"},\\
{"action": "connect OBJ to OBJ", "description": "connect electrical components"},\\
{"action": "disconnect OBJ", "description": "disconnect electrical components"},\\
{"action": "use OBJ [on OBJ]", "description": "use a device/item"},\\
{"action": "look around", "description": "describe the current room"},\\
{"action": "look at OBJ", "description": "describe an object in detail"},\\
{"action": "look in OBJ", "description": "describe a container\'s contents"},\\
{"action": "read OBJ", "description": "read a note or book"},\\
{"action": "move OBJ to OBJ", "description": "move an object to a container"},\\
{"action": "pick up OBJ", "description": "move an object to the inventory"},\\
{"action": "put down OBJ", "description": "drop an inventory item"},\\
{"action": "pour OBJ into OBJ", "description": "pour a liquid into a container"},\\
{"action": "dunk OBJ into OBJ", "description": "dunk a container into a liquid"},\\
{"action": "mix OBJ", "description": "chemically mix a container"},\\
{"action": "go to LOC", "description": "move to a new location"},\\
{"action": "eat OBJ", "description": "eat a food"},\\
{"action": "flush OBJ", "description": "flush a toilet"},\\
{"action": "focus on OBJ", "description": "signal intent on a task object"},\\
{"action": "wait", "description": "take no action for 10 iterations"},\\
{"action": "wait1", "description": "take no action for 1 iteration"},\\
{"action":"examine OBJ","description":"provides a description of the objects present on or in a receptacle."},\\
{"action": "task", "description": "describe current task"},\\
{"action": "inventory", "description": "list your inventory"}$]$\\
\\
Your response should use the following format:\\
Thought:\\
your thoughts.\\
\\
Action:\\
your next action}
\end{tcolorbox}

\begin{tcolorbox}[title=ReAct Prompt Template for TextCraft, colback=gray!5, colframe=black]
\small
\texttt{You are given few useful crafting recipes to craft items in Minecraft. Crafting commands are of the format "craft [target object] using [input ingredients]".\\
Every round I will give you an observation, you have to respond an action based on the state and instruction. You can "get" an object (ingredients) from the inventory or the environment, look-up the game inventory by "inventory", or "craft" (target) using any of the crafting commands.\\
Your output must strictly follow this format: \\
\\
"Thought: \\
your thoughts. \\
\\
Action: \\
your next action" \\
\\
Reminder: \\
1. Always specify the quantity when using "get" and "craft" commands.\\
- Example of get: get 1 lapis lazuli\\
- Example1 of craft: craft 1 blue dye using 1 lapis lazuli\\
- Example2 of craft: craft 1 golden carrot using 8 gold nugget, 1 carrot \\
2. When using "get" command, do not specify whether the item comes from the inventory or the environment.\\
3. You can use ONLY crafting commands provided, do not use your own crafting commands. However, if the crafting command uses a generic ingredient like "planks", you can use special types of the same ingredient e.g. "dark oak planks" in the command instead.
}
\end{tcolorbox}

\begin{tcolorbox}[title=Value Function Prompt Template for WebShop, colback=gray!5, colframe=black]
\small
\texttt{You are a value estimator for a web shopping task.\\
Your job is to estimate the expected future reward (return) from the current STATE, given the INSTRUCTION and the previous STATES, ACTIONS provided as context.\\
\\
Format of the input you will receive:\\
- INSTRUCTION: The original shopping request to evaluate progress against.\\
- HISTORY: Consecutive pairs of (STATE, ACTION) taken so far.\\
- NOW: Current STATE.\\
\\
Use the INSTRUCTION with HISTORY and NOW to infer progress and likelihood of success.\\
Focus on attributes, type, options, and price of shown products relative to the INSTRUCTION.\\
Based on your analysis, provide a score between 0.0 and 1.0.\\
\\
=== INSTRUCTION ===\\
\{Instruction\}\\
\\
=== HISTORY ===\\
\{History\}\\
\\
=== NOW ===\\
\{Current state\}}
\end{tcolorbox}

\begin{tcolorbox}[title=Value Function Prompt Template for SciWorld, colback=gray!5, colframe=black]
\small
\texttt{You are a value estimator for a science world task.\\
Your job is to estimate the expected future reward (return) from the current STATE, given the INSTRUCTION and the previous STATES, ACTIONS provided as context.\\
\\
Format of the input you will receive:\\
- INSTRUCTION: The original scientific request to evaluate progress against.\\
- HISTORY: Consecutive pairs of (STATE, ACTION) taken so far.\\
- NOW: Current STATE.\\
\\
Use the INSTRUCTION with HISTORY and NOW to infer progress and likelihood of success.\\
Focus on milestones, such as locating the correct room, acquiring required tools, and following the specific steps (e.g., 'focus', 'interaction') outlined in the INSTRUCTION.\\
Heavily penalize states where the observation is 'No known action matches that input', as this indicates the agent is stuck in an invalid command loop or syntax error.\\
Based on your analysis, provide a score between -1.0 and 1.0.\\
\\
=== INSTRUCTION ===\\
\{Instruction\}\\
\\
=== HISTORY ===\\
\{History\}\\
\\
=== NOW ===\\
\{Current state\}}
\end{tcolorbox}

\begin{tcolorbox}[title=Value Function Prompt Template for TextCraft, colback=gray!5, colframe=black]
\small
\texttt{You are a value estimator for a Minecraft crafting task.\\
Your job is to estimate the expected future reward (return) from the current STATE, given the INSTRUCTION and the previous STATES, ACTIONS provided as context. \\
\\
Format of the input you will receive: \\
- INSTRUCTION: The original crafting request to evaluate progress against. \\
- HISTORY: Consecutive pairs of (STATE, ACTION) taken so far. \\
- NOW: Current STATE. \\
\\
Use the INSTRUCTION with HISTORY and NOW to infer progress and likelihood of success.\\
Focus on the crafting material, intermediate crafting results and quantities relative to the INSTRUCTION.\\
Based on your analysis, provide a score between 0.0 and 1.0.\\
\\
=== INSTRUCTION ===\\
\{Instruction\}\\
\\
=== HISTORY ===\\
\{History\}\\
\\
=== NOW ===\\
\{Current state\}}
\end{tcolorbox}

\begin{tcolorbox}[title=Zero-shot Prompt Template for WebShop, colback=gray!5, colframe=black]
\small
\texttt{...\textbf{(Preamble identical to WebShop value function template)} \\
First, analyze the real task data provided below.\\
\\
=== INSTRUCTION ===\\
\{Instruction\}\\
\\
=== HISTORY ===\\
\{History\}\\
\\
=== NOW ===\\
\{Current state\}\\
\\
Your response MUST be a single floating-point number between 0.0 and 1.0, and NOTHING ELSE.\\
\\
Score:}
\end{tcolorbox}

\begin{tcolorbox}[title=Zero-shot Prompt Template for SciWorld, colback=gray!5, colframe=black]
\small
\texttt{...\textbf{(Preamble identical to SciWorld value function template)}\\
First, analyze the real task data provided below.\\
\\
=== INSTRUCTION ===\\
\{Instruction\}\\
\\
=== HISTORY ===\\
\{History\}\\
\\
=== NOW ===\\
\{Current state\}\\
\\
Your response MUST be a single floating-point number between -1.0 and 1.0, and NOTHING ELSE.\\
\\
Score:\\}
\end{tcolorbox}

\begin{tcolorbox}[title=Zero-shot Prompt Template for TextCraft, colback=gray!5, colframe=black]
\small
\texttt{...\textbf{(Preamble identical to TextCraft value function template)} \\
=== INSTRUCTION ===\\
\{Instruction\}\\
\\
=== HISTORY ===\\
\{History\}\\
\\
=== NOW ===\\
\{Current state\}\\
\\
Your response MUST be a single floating-point number between 0.0 and 1.0, and NOTHING ELSE. \\
\\
Score:
}
\end{tcolorbox}

\begin{tcolorbox}[title=FoA Prompt Template for WebShop, colback=gray!5, colframe=black]
\small
\texttt{Given an item to purchase and a trajectory that aims to buy an item that exactly matches the specification, which corresponds to the ideal score of 10, analyze the following trajectory, then at the last line conclude "Thus the correctness score is ", where score is an integer from 1 to 10.\\
\\
Here are some examples\\
...\textbf{(Same in-context examples as \citet{klein2025fleetagentscoordinatedproblem})}}\\
\end{tcolorbox}

\begin{tcolorbox}[title=ReflAct Prompt Template for SciWorld, colback=gray!5, colframe=black]
\small
\texttt{You are a helpful assistant to do some scientific experiment in an environment. \\
In the environment, there are several rooms: kitchen, foundry, workshop, bathroom, outside, living room, bedroom, greenhouse, art studio, hallway\\
You should explore the environment and find the items you need to complete the experiment.\\
You can teleport to any room in one step.\\
All containers in the environment have already been opened, you can directly get items from the containers.\\
For each of your turn, you will be given the observation of the last turn.\\
You should first reflect on the agent’s state, including the location, inventory, and focused object, in relation to the task goal. \\
Then, output the action for this turn. Your output must strictly follow this format:\\
\\
Reflection: \\
your reflection.\\
\\
Action: \\
your next action.\\
\\
The available actions are:\\
open OBJ: open a container\\
close OBJ: close a container\\
activate OBJ: activate a device\\
deactivate OBJ: deactivate a device\\
connect OBJ to OBJ: connect electrical components\\
disconnect OBJ: disconnect electrical components\\
use OBJ [on OBJ]: use a device/item\\
look around: describe the current room\\
examine OBJ: describe an object in detail\\
look at OBJ: describe a container’s contents\\
read OBJ: read a note or book\\
move OBJ to OBJ: move an object to a container\\
pick up OBJ: move an object to the inventory\\
pour OBJ into OBJ: pour a liquid into a container\\
mix OBJ: chemically mix a container\\
teleport to LOC: teleport to a specific room\\
focus on OBJ: signal intent on a task object\\
wait: task no action for 10 steps\\
wait1: task no action for a step\\
\\
Example:\\
...\textbf{(Same in-context examples as \citet{kim2025reflactworldgroundeddecisionmaking})}}\\
\end{tcolorbox}

\section{Additional Experiments}
\label{app:experiments}
\subsection{Resampling Steps}
\label{app:resampling}
\subsubsection{Task-Specific Fixed-Step Resampling}\label{app:resampling_fixed}

\begin{figure}[htbp]
    \centering
    \includegraphics[width=0.6\textwidth]{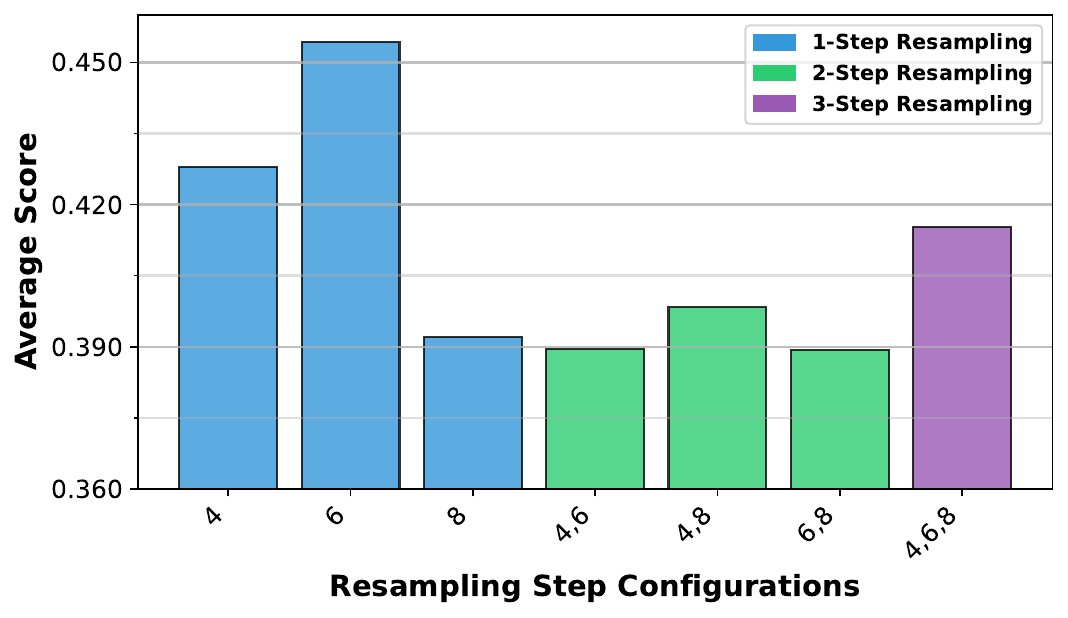}
    \caption{WebShop performance across different resampling steps.}
    \label{fig:resample_webshop}
\end{figure}

\begin{figure}[htbp]
    \centering
    \includegraphics[width=0.7\textwidth]{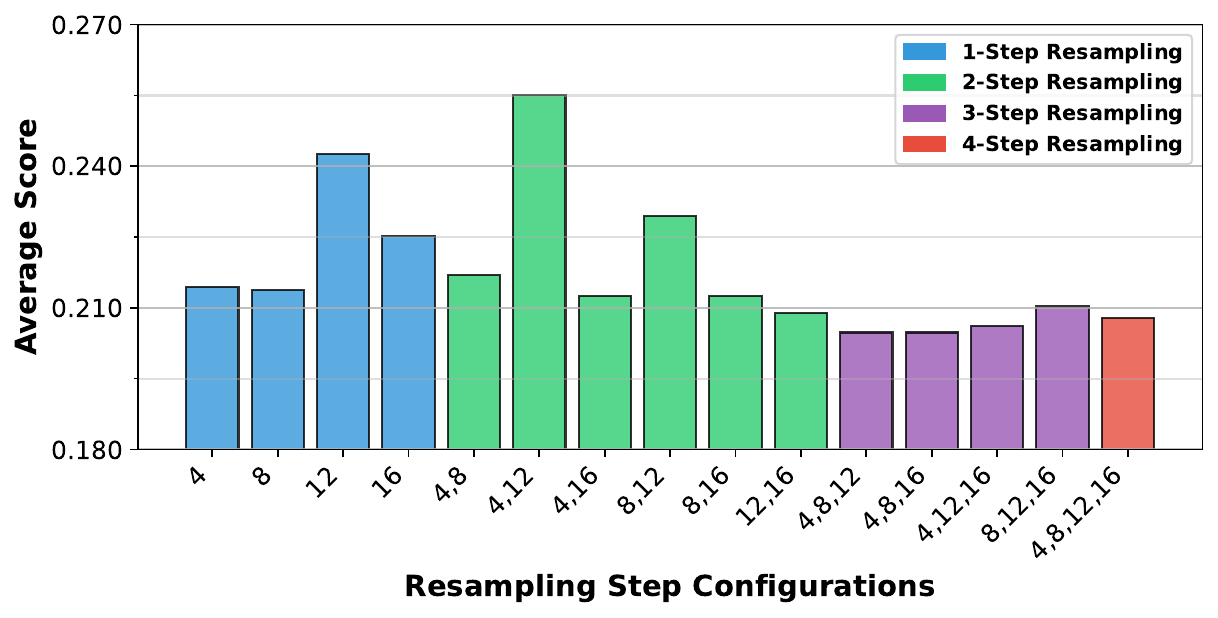}
    \caption{SciWorld performance across different resampling steps.}
    \label{fig:resample_sciworld}
\end{figure}

\begin{figure}[htbp]
    \centering
    \includegraphics[width=0.7\textwidth]{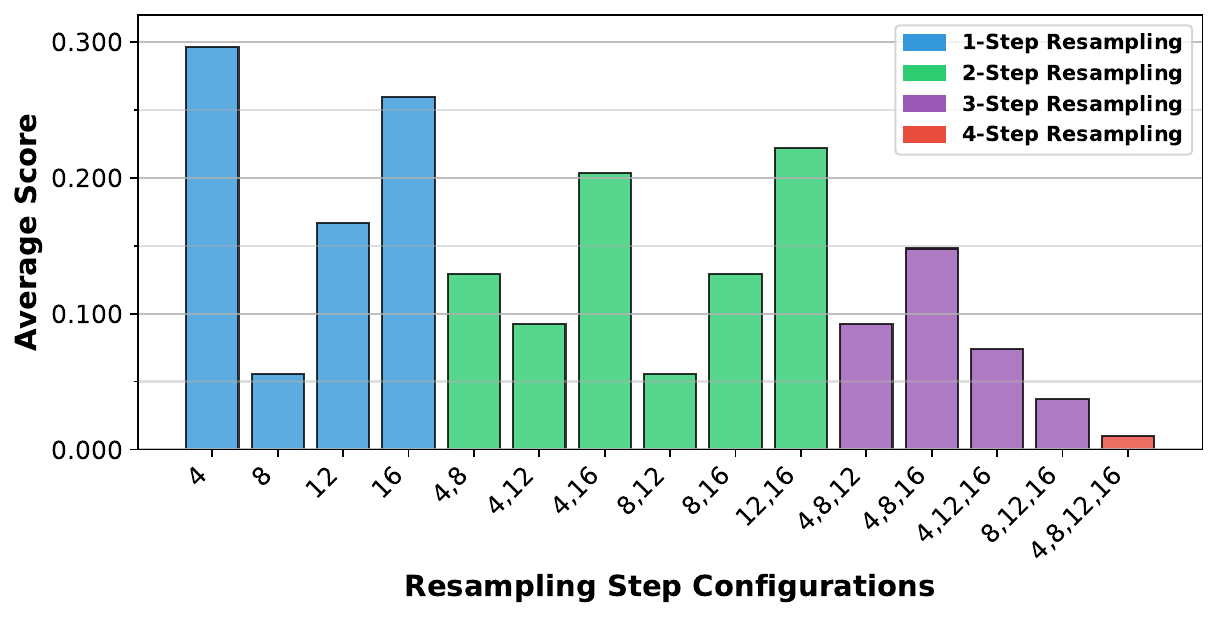}
    \caption{TextCraft performance across different resampling steps.}
    \label{fig:resample_textcraft}
\end{figure}

Our ablation study on resampling step configurations with 5 trajectories (\autoref{fig:resample_webshop}, \autoref{fig:resample_sciworld}, and \autoref{fig:resample_textcraft}) demonstrates that agent performance is highly adaptable to the strategic timing of value function interventions. Across all tested environments, results indicate that increasing resampling frequency does not inherently guarantee higher performance. Instead, AMC’s strength lies in its ability to provide high-impact, sparse corrections at pivotal moments. For example, in TextCraft, a dense 4-step resampling approach actually hindered performance compared to more focused 1-step or 2-step interventions. Peak efficiency is reached through task-specific timing: a single intervention at step 6 for WebShop, a dual-step approach at steps 4 and 12 for SciWorld, and an early correction at step 4 for TextCraft. These results highlight that AMC’s effectiveness is driven by precision-targeted interventions rather than sheer computational volume, allowing for significant performance gains while maintaining minimal overhead.

\subsubsection{Dynamic Resampling with Effective Sample Size}\label{app:resampling_ess}

Empirical analysis in \appref{app:resampling_fixed} reveals that the optimal resampling frequency is inherently task-dependent. For instance, WebShop requires early-to-mid trajectory interventions to preserve the enough step budget to finalize transactions. To confirm this, we compare fixed-step resampling against a dynamic criterion based on the Effective Sample Size (ESS) \cite{doucet2000sequential}. Specifically, we resample at step $t$ using normalized weights $\tilde w_t^{(i)}$ whenever $ESS := (\sum_{i=1}^{N}(\tilde{w}_t^{(i)})^2)^{-1} < N\rho$ for a given threshold $\rho\in (0,1)$. While ESS is designed to mitigate trajectory degeneracy by filtering low-weight samples, \autoref{tab:ablation_ess} shows that fixed-step selections consistently outperform ESS across various thresholds on both WebShop and SciWorld. Beyond superior performance, fixed-step resampling is significantly more computationally efficient. While ESS requires value network evaluation at every step, fixed-step allows the recursive update rule (\autoref{eq:importance_weight_recursive}) to telescope, effectively skipping value estimation at most steps.

\begin{table}[h] 
\centering
\caption{Comparison of dynamic resampling and fixed-step resampling for AMC.}
\scalebox{0.8}{
\begin{tblr}{
  colspec = {l l l l c}, 
  vlines, 
  hline{1,2,7,12} = {-}{}, 
  cell{2,7}{1,2} = {r=5}{l, m}, 
  cell{3,8}{3} = {r=4}{l, m},
  cell{2,7}{3-5} = {bg=orange!10}, 
  cell{3-6,8-11}{3-5} = {bg=green!10},
  row{1} = {font=\bfseries},
  row{1-Z} = {rowsep=2pt} 
}
Dataset & Policy Model & Method & Resampling ($\rho$) & Score \\
WebShop & Llama-3.2-11B & Best-of-5 & - & 0.385 ($\pm$0.015) \\
        &               & AMC       & ESS (0.1)   & 0.387 ($\pm$0.015) \\
        &               &           & ESS (0.5)   & 0.402 ($\pm$0.017) \\
        &               &           & ESS (0.9)   & \textbf{0.435 ($\pm$0.013)} \\
        &               &           & Fixed-step & \textbf{0.454 ($\pm$0.009)} \\
SciWorld & Llama-3.1-8B & Best-of-5 & - & 0.206 ($\pm$0.010) \\
         &              & AMC       & ESS (0.1)   & 0.211 ($\pm$0.007) \\
         &              &           & ESS (0.5)   & 0.216 ($\pm$0.006) \\
         &              &           & ESS (0.9)   & 0.228 ($\pm$0.006) \\
         &              &           & Fixed-step & \textbf{0.255 ($\pm$0.008)} \\
\end{tblr}
}
\label{tab:ablation_ess}
\end{table}\label{tab:value_estimation}

\begin{table}[h]
\centering
\caption{Comparison of value estimation objectives when $N=15$.}
\scalebox{0.85}{
\begin{tblr}{
  colspec = {l l l c}, 
  vlines, 
  hline{1,2,3,4} = {-}{}, 
  cell{2}{1,2} = {r=2}{l, m}, 
  row{1} = {font=\bfseries},
  row{1-Z} = {rowsep=2pt} 
}
Dataset & Policy Model & Sampling & Score \\
SciWorld & Llama-3.1-8B & \autoref{eq:loss_func_real} & 0.347 ($\pm$ 0.015) \\
         &              & \autoref{eq:loss_func_ideal} & 0.332 ($\pm$ 0.005) \\
\end{tblr}
}
\label{tab:ablation_sampling}
\end{table}

\subsection{Value Estimation Approaches}
\label{app:value_estimation}

As discussed in \autoref{sec:learning_value}, we approximate the ideal value $V$ in  \autoref{eq:loss_func_ideal}, which requires a soft maximum over multiple trajectories per task, with the single trajectory formulation in \autoref{eq:loss_func_real}. Multiple sampling has also been used in MCMC frameworks for uncertainty quantification in LLMs \cite{ross2026textual}. To evaluate the empirical impact of this approximation bias, we compare the performance of both methods on SciWorld in \autoref{tab:ablation_sampling}. For \autoref{eq:loss_func_ideal}, we use 3 trajectory estimates to better approximate the log-sum-exp.
We find that the single trajectory approximation yields comparable performance to the multiple trajectory baseline. This indicates that \autoref{eq:loss_func_real} serves as an effective surrogate for \autoref{eq:loss_func_ideal}. In the context of our training set construction approach (\autoref{app:setup}), in which we sample multiple trajectories per task and only include the max-reward trajectory, one can view \autoref{eq:loss_func_real} as using a hard maximum to approximate the soft maximum of \autoref{eq:loss_func_ideal}.

\subsection{Impact of Base Model Selection on Value Function}
\label{app:base_model}
We evaluated the AMC framework using an additional base model architecture for the value model, Qwen-3-4B, alongside open-weight policies, to verify that the value function's effectiveness is not dependent on a specific base model from the Llama-3 family. As illustrated in \autoref{tab:generalizability}, a Qwen-3-4B based value model achieves competitive performance across all benchmarks, closely trailing the results of the Llama-3 family. This confirms the architectural generalizability of the value function and suggests that overall performance can be further optimized by selecting the base model best suited to a specific benchmark.

\begin{table}[h]
\centering
\caption{Sensitivity of the value function backbone across benchmarks with 5 trajectories.}
\scalebox{0.8}{
\begin{tblr}{
  colspec = {l l l c},
  vlines,
  hline{1,2,4,6,8} = {-}{}, 
  cell{2,4,6}{1} = {r=2}{l}, 
  cell{2,4,6}{2} = {r=2}{l},
}
\textbf{Dataset} & \textbf{Policy Model} & \textbf{Value Model} & \textbf{Score} \\
WebShop   & Llama-3.2-11B & Llama-3.2-11B & 0.454 ($\pm$0.014) \\
          &               & Qwen-3-4B     & 0.444 ($\pm$0.017) \\         
SciWorld  & Llama-3.1-8B  & Llama-3.1-8B  & 0.249 ($\pm$0.011) \\
          &               & Qwen-3-4B     & 0.236 ($\pm$0.011) \\
TextCraft & Llama-3.2-11B & Llama-3.2-11B & 0.234 ($\pm$ 0.012) \\
          &               & Qwen-3-4B     & 0.222 ($\pm$ 0.043) \\
\end{tblr}}
\label{tab:generalizability}
\end{table}

\subsection{GRPO Baseline and Cost Analysis}
\label{app:grpo_cost}

Since \citet{agentgym-rl} did not release their GRPO models, we used a 8xA100 40GB node to fine-tune our own GRPO baselines for 2 epochs, removing a leak of 30\% of test tasks identified in the AgentGym-RL codebase. As shown in \autoref{tab:grpo_cost_ours}, GRPO with Best-of-$N$ provides only marginal improvements over GRPO across different policies, as the GRPO policy is already converged with a low standard error across 3 trials. When considering GPT-5.1 with a Qwen-2.5-7B value model, AMC significantly outperforms GRPO models with Best-of-$N$, showing the practicality of implementing our approach on a black-box policy. When considering the same policy model as GRPO (i.e., Qwen-2.5-3B), we outperform GRPO at the same $N$ while saving on total training and inference costs. Note that cost comparisons use AWS on-demand rates as of April 2026 (USD 21.96/hr for 8xA100, which is used only for GRPO training; USD 5.27/hr for an NVIDIA RTX 6000 96GB, which is the most comparable currently available server GPU to our main evaluation environment).

\begin{table*}[h] 
\centering
\caption{Performance and cost efficiency comparison between GRPO and AMC across varying numbers of trajectories.}
\scalebox{0.85}{
\begin{tblr}{
colspec = {l l l l c c c}, 
  vlines,
  hline{1,2,6,11} = {-}{},  
  hline{5,9} = {1}{solid},
  cell{2,6}{1} = {r=3}{l, m}, 
  cell{5}{1}   = {l, m},      
  cell{9}{1}   = {r=2}{l, m}, 
  cell{2}{3} = {l, m, bg=gray!10}, 
  cell{3}{3} = {r=2}{l, m, bg=orange!10}, 
  cell{5}{3} = {l, m, bg=green!10}, 
  cell{6}{3} = {l, m, bg=gray!10}, 
  cell{7}{3} = {r=2}{l, m, bg=orange!10},
  cell{9}{3} = {r=2}{l, m, bg=green!10}, 
  cell{2}{4} = {l, m, bg=gray!10}, 
  cell{3}{4} = {r=2}{l, m, bg=orange!10},
  cell{5}{4} = {l, m, bg=green!10},
  cell{6}{4} = {l, m, bg=gray!10},
  cell{7}{4} = {r=2}{l, m, bg=orange!10},
  cell{9}{4} = {r=2}{l, m, bg=green!10},
  cell{2,6}{2,5,6,7} = {bg=gray!10},
  cell{3,4,7,8}{2,5,6,7} = {bg=orange!10},
  cell{5,9,10}{2,5,6,7} = {bg=green!10},
  row{1} = {font=\bfseries},
  row{1-Z} = {rowsep=2pt}
}
Training & Method & Policy Model & Value Model & Score & Training Cost (USD) & Inference Cost (USD) \\
GRPO & ReAct & Qwen-2.5-7B & - & 0.183 ($\pm$ 0.001) & 702.64 & \textbf{10.54} \\
     & Best-of-5 & Qwen-2.5-7B&- & 0.190 ($\pm$ 0.001) & 702.64 & 55.32 \\
     & Best-of-20 & & & 0.194 ($\pm$ 0.001) & 702.64 & 231.80 \\
-    & AMC ($N=5$) & GPT-5.1 & Qwen-2.5-7B & \textbf{0.518 ($\pm$ 0.030)} & \textbf{138.39} & 57.97 \\
GRPO & ReAct & Qwen-2.5-3B & - & 0.182 ($\pm$ 0.001) & 395.24 & \textbf{5.27} \\
     & Best-of-5 & Qwen-2.5-3B&- & 0.185 ($\pm$ 0.003) & 395.24 & 12.28 \\
     & Best-of-20 & & & 0.187 ($\pm$ 0.001) & 395.24 & 50.05 \\
-    & AMC ($N=5$) &Qwen-2.5-3B & Qwen-2.5-3B & 0.133 ($\pm$ 0.006) & 147.51 & 47.41 \\
     & AMC ($N=20$) & & & \textbf{0.216 ($\pm$ 0.004)} & \textbf{147.51} & 210.73 \\
\end{tblr}
}
\label{tab:grpo_cost_ours}
\end{table*}

\subsection{Evaluating for Average Trajectory Performance}\label{app:eval_average}

Our primary evaluation relies on the highest-reward trajectory. This is the most practical metric, as downstream applications typically deploy only the single best outcome. However, to ensure our results are robust against high-variance noise, we expand our evaluation to include an alternative metric. In  \autoref{tab:ablation_average}, we evaluate the mean reward of the top 5 trajectories (out of 15 generated) for each task, averaged these results across all tasks in the dataset. Using Llama-based policies across WebShop, SciWorld, and TextCraft, AMC consistently outperforms the  average top-5-of-15 baseline (derived from the Best-of-15). These results indicate that AMC significantly elevates the overall quality of the top-performing trajectories within the trajectory pool, rather than relying on isolated high-reward outliers to achieve a superior score. These results are corroborated by the analysis in \appref{app:amc_average}, in which we show that, on expectation, AMC is guaranteed to improve the average reward across trajectories.

\begin{table}[h] 
\centering
\caption{Comparison of the averaged top 5 highest reward trajectories.}
\scalebox{0.85}{
\begin{tblr}{
  colspec = {l l c}, 
  vlines, 
  hline{1,2,4,6,8} = {-}{}, 
  cell{2,4,6}{2-3} = {bg=orange!10},   
  cell{3,5,7}{2-3} = {bg=green!10}, 
  cell{2,4,6}{1} = {r=2}{l, m}, 
  row{1} = {font=\bfseries},
  row{1-Z} = {rowsep=2pt} 
}
Dataset & Method & Score \\
WebShop & Best-of-15 (Top-5) & 0.202 ($\pm$ 0.014) \\
        & AMC (Top-5)        & \textbf{0.371 ($\pm$ 0.013)} \\
SciWorld & Best-of-15 (Top-5) & 0.183 ($\pm$ 0.010) \\
         & AMC (Top-5)        & \textbf{0.226 ($\pm$ 0.014)} \\
TextCraft & Best-of-15 (Top-5) & 0.215 ($\pm$ 0.030) \\
          & AMC (Top-5)        & \textbf{0.462 ($\pm$ 0.020)} \\         
\end{tblr}
}
\label{tab:ablation_average}
\end{table}

\subsection{Additional Datasets}
\label{app:additional_datasets}

To further evaluate the robustness of our approach, we test AMC on two additional tasks from the AgentBoard benchmark \cite{ma2024agentboardanalyticalevaluationboard}. Unlike our previous benchmarks, these datasets emphasize tool use, temporal synthesis, and multi-hop relational reasoning:

\begin{itemize}
    \item \textbf{Weather}: This task evaluates the agent's ability to perform multi-step information retrieval using a specialized Weather API. The agent navigates 18 distinct actions (e.g., querying historical data, local forecasts) to answer complex queries such as comparing precipitation across different time periods. Success requires precise tool-calling and temporal logic. We used a strict binary reward (1 only for an exact match with the ground truth), a maximum horizon of $T=10$, and a resampling step of 4.
    
    \item \textbf{Movie}: This dataset tests the agent's ability to perform relational queries within a large-scale cinematic database. With 16 available actions, the agent often executes multi-hop reasoning (e.g., identifying a specific director's filmography, filtering by release year, and then retrieving the lead actor's accolades). We utilized a strict binary reward (1 only for an exact match with the ground truth), a maximum horizon of $T=12$, and a resampling step of 6.
\end{itemize}
As shown in \autoref{tab:weather_movie_result}, AMC consistently matches or outperforms the Best-of-$N$ and SMC (Zero-shot) baselines across these new domains. These results demonstrate that AMC's performance gains are robust and generalizable to diverse tasks with varying requirements.

\begin{table}[h] 
\centering
\caption{Weather and Movie performance, obtained with a Llama-3.2-11B-based value model for SMC (Zero-shot) and AMC.}
\scalebox{0.85}{
\begin{tblr}{
  colspec = {l l l c},
  vlines,
  hline{1,2,23,Z} = {-}{solid},
  hline{9,16,30,37} = {2-4}{solid},
  hline{6,13,20,27,34,41} = {3-4}{solid},
  cell{2,23}{1} = {r=21}{m},
  cell{2,9,16,23,30,37}{2} = {r=7}{m},
  cell{2,9,16,23,30,37}{3,4} = {bg=gray!10},
  cell{3,4,10,11,17,18,24,25,31,32,38,39}{3,4} = {bg=orange!10},
  cell{5,12,19,26,33,40}{3,4} = {bg=green!10},
  cell{6,7,13,14,20,21,27,28,34,35,41,42}{3,4} = {bg=orange!10},
  cell{8,15,22,29,36,43}{3,4} = {bg=green!10},
  row{1} = {font=\bfseries},
  row{1-Z} = {rowsep=1pt}
}
Dataset & Policy Model & Method & Score \\
Weather & Llama-3.2-11B & ReAct & 0.433 ($\pm$ 0.041) \\
        &               & Best-of-5 & \textbf{0.517 ($\pm$ 0.020)} \\
        &               & SMC (Zero-shot) & \textbf{0.483 ($\pm$ 0.054)} \\
        &               & AMC & \textbf{0.533 ($\pm$ 0.020)} \\
        &               & Best-of-15 & \textbf{0.550 ($\pm$ 0.000)} \\
        &               & SMC (Zero-shot) & \textbf{0.550 ($\pm$ 0.000)} \\
        &               & AMC & \textbf{0.550 ($\pm$ 0.000)} \\
        & GPT-4.1-mini  & ReAct & 0.450 ($\pm$ 0.035) \\
        &               & Best-of-5 & 0.483 ($\pm$ 0.041) \\
        &               & SMC (Zero-shot) & 0.517 ($\pm$ 0.020) \\
        &               & AMC & \textbf{0.550 ($\pm$ 0.000)} \\
        &               & Best-of-15 & \textbf{0.533 ($\pm$ 0.020)} \\
        &               & SMC (Zero-shot) & \textbf{0.550 ($\pm$ 0.000)} \\
        &               & AMC & \textbf{0.550 ($\pm$ 0.000)} \\
        & GPT-5.1       & ReAct & 0.550 ($\pm$ 0.000) \\
        &               & Best-of-5 & \textbf{0.550 ($\pm$ 0.000)} \\
        &               & SMC (Zero-shot) & \textbf{0.550 ($\pm$ 0.000)} \\
        &               & AMC & \textbf{0.567 ($\pm$ 0.020)} \\
        &               & Best-of-15 & 0.550 ($\pm$ 0.000) \\
        &               & SMC (Zero-shot) & 0.550 ($\pm$ 0.000) \\
        &               & AMC & \textbf{0.583 ($\pm$ 0.020)} \\
Movie   & Llama-3.2-11B & ReAct & 0.417 ($\pm$ 0.041) \\
        &               & Best-of-5 & \textbf{0.833 ($\pm$ 0.041)} \\
        &               & SMC (Zero-shot) & \textbf{0.767 ($\pm$ 0.054)} \\
        &               & AMC & \textbf{0.783 ($\pm$ 0.020)} \\
        &               & Best-of-15 & 0.850 ($\pm$ 0.000) \\
        &               & SMC (Zero-shot) & \textbf{0.850 ($\pm$ 0.035)} \\
        &               & AMC & \textbf{0.883 ($\pm$ 0.020)} \\
        & GPT-4.1-mini  & ReAct & 0.750 ($\pm$ 0.061) \\
        &               & Best-of-5 & 0.800 ($\pm$ 0.035) \\
        &               & SMC (Zero-shot) & 0.817 ($\pm$ 0.020) \\
        &               & AMC & \textbf{0.850 ($\pm$ 0.000)} \\
        &               & Best-of-15 & \textbf{0.850 ($\pm$ 0.000)} \\
        &               & SMC (Zero-shot) & \textbf{0.850 ($\pm$ 0.000)} \\
        &               & AMC & \textbf{0.850 ($\pm$ 0.000)} \\
        & GPT-5.1       & ReAct & 0.867 ($\pm$ 0.020) \\
        &               & Best-of-5 & \textbf{0.883 ($\pm$ 0.020)} \\
        &               & SMC (Zero-shot) & 0.867 ($\pm$ 0.020) \\
        &               & AMC & \textbf{0.900 ($\pm$ 0.000)} \\
        &               & Best-of-15 & \textbf{0.883 ($\pm$ 0.020)} \\
        &               & SMC (Zero-shot) & \textbf{0.900 ($\pm$ 0.000)} \\
        &               & AMC & \textbf{0.900 ($\pm$ 0.000)} \\
\end{tblr}
}
\label{tab:weather_movie_result}
\end{table}

\section{Qualitative Analysis}
\label{app:qualitative_analysis}
\autoref{fig:example_webshop} provides a side-by-side comparison of two AMC trajectories on  WebShop, highlighting a critical divergence between the two at resampling step 6. In Case 1, the agent identifies a product candidate ([B09LM3H2F5]) that initially satisfies several key semantic and budgetary criteria. This partial alignment yields a high value score of 0.6, which encourages the agent to proceed with this trajectory, ultimately securing a successful 1.0 reward through early task completion. In contrast, Case 2 illustrates a representative failure mode. While the agent initiates the task with the same search query as Case 1, it fails to prioritize relevant items, eventually returning to the initial search interface at step 6. At this step, our value model assigns a low score of 0.1, identifying the trajectory as being far from the target state, from which it goes on to achieve a reward of 0.0. This qualitative comparison confirms that our value function effectively distinguishes between goal-aligned progress and unproductive behavior at resampling step 6.

The value model also effectively distinguishes between promising and non-promising trajectories within the SciWorld environment. For the qualitative evaluation, we focus on resampling step 12, the second-most effective intervention from our ablation study (\autoref{fig:resample_sciworld}), as it provides a clearer baseline for analyzing individual trajectory outcomes. As illustrated in \autoref{fig:example_sciworld} (left), by step 12, the agent has successfully located its target—a non-living object (finger painting). At this stage, the value model assigns a high score of 0.4 out of its range of $[-1, 1]$. This trajectory culminates at step 15, where the agent enters the bathroom and places the painting in the red box, achieving a perfect environment reward of 1.0. Conversely, the non-promising trajectory in \autoref{fig:example_sciworld} (right) illustrates the agent repeatedly performing redundant actions to access the bathroom despite the absence of a direct path from the hallway. This lack of progress results in a low value function score of 0.1 at step 12. Although the agent eventually identifies the target non-living item at step 18 and attempts to pick it up in the final step, it ultimately receives an environment reward of 0.0. This failure stems from the agent's inability to satisfy necessary sub-goals, such as the initial ``focus on" action required by the objective.

Furthermore, the value function can accurately provide more granular comparisons between relatively similar states. \autoref{fig:example_textcraft} presents a qualitative analysis of two partial trajectories within the TextCraft domain. While both trajectories align with the optimal policy, Case 1 has progressed further along the crafting dependency tree compared to Case 2. Reflecting this state progression, the value function assigns a higher value estimate to the former (0.275) than the latter (0.191). This demonstrates that the value function can identify states closer to the terminal condition possess a higher probability of success.

\begin{figure}[h]
\centering
\caption{Comparison of agent trajectories in WebShop. We highlight the divergence at the resampling step 6, where the value function differentiates between promising and non-promising ones.}
\small
\begin{minipage}{0.49\textwidth}
    \begin{tcolorbox}[colback=white, colframe=black, arc=0mm, title=Case 1: Promising (Success), height=16.5cm]
        \textbf{Goal:} Find me hand crafted candy \& chocolate for gift set, valentine day with style: happy birthday truffle chocolate, and price lower than 60.00 dollars.
        \vspace{1mm} \hrule \vspace{2mm}
        \textbf{[STEP 1]}\\
        \textbf{Action:}\\
        \texttt{search[hand crafted candy gift set valentine day truffle chocolate birthday price:60]}\\
        
        \dots (Search around items)        

        \vspace{3mm}
        \textbf{[STEP 6]}\\ 
        \textbf{Observation:} Item $[$B09LM3H2F5$]$ partially matches gift set, truffle, and $<$ \$60. \\
        \textbf{Value Function: 0.6}\\
        \textbf{Action:} \texttt{click[B09LM3H2F5]}
        
        \vspace{3mm}
        \textbf{[STEP 7]} \\
        \textbf{Observation:} Options for style refinement: \textit{happy birthday truffle chocolate, holiday truffle chocolate, truffle collection chocolate}. \\
        \textbf{Action:} \texttt{click[style]}
        
        \vspace{3mm}
        \textbf{[STEP 8]} \\
        \textbf{Observation:} Specific style options available. \\        
        \textbf{Action:} \\
        \texttt{click[happy birthday truffle chocolate]}
       
        \vspace{3mm}
        \textbf{[STEP 9]}\\
        \textbf{Observation:} Identified product matching all criteria. \\
        \textbf{Action:} \texttt{click[Buy Now]}\\        
        \vfill
        \hrule \vspace{1mm}
        \textbf{Environment Reward: 1.0 (Success)}
        
    \end{tcolorbox}
\end{minipage}
\hfill
\begin{minipage}{0.49\textwidth}
    \begin{tcolorbox}[colback=white, colframe=black, arc=0mm, title=Case 2: Non-Promising (Failure), height=16.5cm]
        \textbf{Goal:} Find me hand crafted candy \& chocolate for gift set, valentine day with style: happy birthday truffle chocolate, and price lower than 60.00 dollars.
        \vspace{1mm} \hrule \vspace{2mm}
        \textbf{[STEP 1]}\\
        \textbf{Action:}\\
        \texttt{search[hand crafted candy chocolate gift set valentine happy birthday truffle price lower 60]}\\
        
        \dots (Search around items)        

        \vspace{3mm}
        \textbf{[STEP 6]} \\
        \textbf{Observation:} Back to the first page to search. \\
        \textbf{Value Function: 0.1} \\
        \textbf{Action:} \\
        \texttt{search[hand crafted candy chocolate gift set valentine happy birthday truffle price lower 60]}
        
        \vspace{3mm}
        \textbf{[STEP 7]} \\
        \textbf{Observation:} Search results for previous request. \\
        \textbf{Action:} \\
        \texttt{search[hand crafted candy chocolate gift set valentine happy birthday truffle chocolate gift box]}
        
        \vspace{3mm}
        \textbf{[STEP 8]}\\
        \textbf{Observation:} Search results for previous request. \\
        \textbf{Action:} \\
        \texttt{click[Waterbridge Wave Milk Chocolate Hazelnut Crunch]}
        
        \vspace{3mm}
        \textbf{[STEP 9]}\\
        \textbf{Observation:} View product page for Waterbridge Wave... \\
        \textbf{Action:} \texttt{click[Add to Cart]}
               
        \vspace{3mm}
        \textbf{[STEP 10]}\\
        \textbf{Observation:} No state transition.\\
        \textbf{Action:} \\
        \texttt{click[CRAVINGS BY ZOE Valentine’s Day Chocolate Truffles]}\\
        \vfill
        \hrule \vspace{1mm}
        \textbf{Environment Reward: 0.0 (Failure)}
        
    \end{tcolorbox}
\end{minipage}
\label{fig:example_webshop}
\end{figure}

\begin{figure}[H]
\centering
\caption{Comparison of agent trajectories in SciWorld. We show the divergence at the resampling step 12, where the value function distinguishes between promising and non-promising ones.}
\small
\begin{minipage}{0.49\textwidth}
    \begin{tcolorbox}[colback=white, colframe=black, arc=0mm, title=Case 1: Promising (Success), height=18.5cm]
        \textbf{Goal:} Your task is to find a(n) non-living thing. First, focus on the thing. Then, move it to the red box in the bathroom.\\
        This room is called the hallway. In it, you see:\\ 
	    the agent, a substance called air, a finger painting\\
        You also see:\\
    	A door to the kitchen (that is closed)\\
        
        \dots (Multiple doors to alternative locations are visible, but a door for the bathroom is absent)        
        \vspace{1mm} \hrule \vspace{2mm}
        \textbf{[STEP 1]}\\
        \textbf{Action:} \texttt{Focus on the finger painting}
        
        \vspace{3mm}
        \textbf{[STEP 2]}\\ 
        \textbf{Observation:} You focus on the finger painting. \\        
        \textbf{Action:} \texttt{pick up the finger painting}\\

        \dots (Explore the hallway)

        \vspace{3mm}
        \textbf{[STEP 9]}\\ 
        \textbf{Observation:} You move through the door to the kitchen.\\  
        \textbf{Action:} \texttt{look around}

        \vspace{3mm}
        \textbf{[STEP 10]}\\ 
        \textbf{Observation:} The description of the kitchen. There is a door to the bathroom (that is closed)\\        
        \textbf{Action:} \texttt{go to LOC (bathroom)}

        \vspace{3mm}
        \textbf{[STEP 11]}\\ 
        \textbf{Observation:} No known action matches that input.\\        
        \textbf{Action:} \texttt{open bathroom door}

        \vspace{3mm}
        \textbf{[STEP 12]}\\ 
        \textbf{Observation:} The door is now open.\\   
        \textbf{Value Function: 0.4}\\
        \textbf{Action:} \texttt{look around}\\

        \dots (Explore the kitchen)

        \vspace{3mm}
        \textbf{[STEP 14]}\\ 
        \textbf{Observation:} You move through the door to the bathroom.\\    
        \textbf{Action:} \texttt{look around}

        \vspace{3mm}
        \textbf{[STEP 15]}\\ 
        \textbf{Observation:} The description of the bathroom. There is a red box (containing nothing)\\        
        \textbf{Action:} \texttt{move finger painting to red box}\\        
        \vfill
        \hrule \vspace{1mm}
        \textbf{Environment Reward: 1.0 (Success)}
        
    \end{tcolorbox}
\end{minipage}
\hfill
\begin{minipage}{0.49\textwidth}
    \begin{tcolorbox}[colback=white, colframe=black, arc=0mm, title=Case 2: Non-Promising (Failure), height=18.5cm]
        \textbf{Goal:} Your task is to find a(n) non-living thing. First, focus on the thing. Then, move it to the red box in the bathroom.\\
        This room is called the hallway. In it, you see:\\ 
	    the agent, a substance called air, a finger painting\\
        You also see:\\
    	A door to the kitchen (that is closed)\\
        
        \dots (Multiple doors to alternative locations are visible, but a door for the bathroom is absent)       
        \vspace{1mm} \hrule \vspace{2mm}
        \textbf{[STEP 1]}\\
        \textbf{Action:} \texttt{look around}
        
        \vspace{3mm}
        \textbf{[STEP 2]}\\ 
        \textbf{Observation:} The description of the hallway. \\        
        \textbf{Action:} \texttt{go to LOC (bathroom)}

        \vspace{3mm}
        \textbf{[STEP 3]}\\ 
        \textbf{Observation:} No known action matches that input.\\  
        \textbf{Action:} \texttt{go to LOC (bathroom)}\\

        \dots (Redundant actions attempted to access the bathroom)

        \vspace{3mm}
        \textbf{[STEP 12]}\\ 
        \textbf{Observation:} No known action matches that input.\\  
        \textbf{Value Function: 0.1}\\
        \textbf{Action:} \texttt{go to LOCATION bathroom}\\

        \dots (Redundant actions attempted to access the bathroom)

        \vspace{3mm}
        \textbf{[STEP 17]}\\ 
        \textbf{Observation:} No known action matches that input.\\    
        \textbf{Action:} \texttt{look around}

        \vspace{3mm}
        \textbf{[STEP 18]}\\ 
        \textbf{Observation:} The description of the hallway.\\    
        \textbf{Action:} \texttt{look at finger painting}
        
        \vspace{3mm}
        \textbf{[STEP 19]}\\ 
        \textbf{Observation:} a finger painting of a cow. The artist is listed as Owen.\\    
        \textbf{Action:} \texttt{move finger painting to bathroom}

        \vspace{3mm}
        \textbf{[STEP 20]}\\ 
        \textbf{Observation:} No known action matches that input.\\        
        \textbf{Action:} \texttt{pick up finger painting}\\        
        \vfill
        \hrule \vspace{1mm}
        \textbf{Environment Reward: 0.0 (Failure)}
    \end{tcolorbox}
\end{minipage}
\label{fig:example_sciworld}
\end{figure}

\begin{figure}[H]
\centering
\caption{Comparison of agent trajectories in TextCraft. We highlight the divergence at the resampling step 4, where the value function differentiates between more promising and less promising trajectories.}
\small
\begin{minipage}{0.49\textwidth}
    \begin{tcolorbox}[colback=white, colframe=black, arc=0mm, title=Case 1: Promising (Success), height=19.5cm]
        Crafting commands: \\
        craft 1 lever using 1 cobblestone, 1 stick \\
        craft 6 andesite slab using 3 andesite \\
        craft 6 andesite wall using 6 andesite \\ 
        craft 6 cobblestone wall using 6 cobblestone \\
        craft 2 diorite using 2 quartz, 2 cobblestone \\
        craft 4 andesite stairs using 6 andesite \\
        craft 1 granite using 1 diorite, 1 quartz \\
        craft 1 mossy cobblestone using 1 cobblestone, 1 vine \\
        craft 1 quartz block using 4 quartz \\
        craft 1 dropper using 1 redstone, 7 cobblestone \\
        craft 1 dispenser using 1 redstone, 7 cobblestone, 1 bow \\
        craft 4 diorite stairs using 6 diorite \\
        craft 2 andesite using 1 diorite, 1 cobblestone \\
        \textbf{Goal:} craft andesite wall.
        \vspace{1mm} \hrule \vspace{2mm}
        
        \textbf{[STEP 1]}\\
        \textbf{Action:} \\
        \texttt{Inventory} \\
        \\
        \textbf{[STEP 2]}\\
        \textbf{Observation:} \\
        \texttt{Inventory: You are not carrying anything.} \\
        \textbf{Action:} \\
        \texttt{Get 2 quartz \\ Get 2 cobblestone} \\
        \\
        \textbf{[STEP 3]}\\
        \textbf{Observation:} \\
        \texttt{Got 2 quartz.} \\
        \textbf{Action:} \\
        \texttt{Get 2 cobblestone} \\
                \\
        \textbf{[STEP 4]}\\
        \textbf{Observation:} \\
        \texttt{Got 2 cobblestone.} \\
        \textbf{Value Function: 0.275} \\ 
        \textbf{Action:} \\
        \texttt{craft 2 diorite using 2 quartz, 2 cobblestone}

        \dots \\ \\
        \textbf{[STEP 16]}\\
        \textbf{Observation:}  \\
        \texttt{Could not execute [no action].} \\
        \\ 
        \textbf{Action:} \\
        \texttt{craft 6 andesite wall using 6 andesite} \\ \\
        \hrule \vspace{1mm}
        \textbf{Environment Reward: 1.0 (Success)}
        
    \end{tcolorbox}
\end{minipage}
\hfill
\begin{minipage}{0.49\textwidth}
    \begin{tcolorbox}[colback=white, colframe=black, arc=0mm, title=Case 2: Non-Promising (Failure), height=19.5cm]
        Crafting commands: \\
        craft 1 lever using 1 cobblestone, 1 stick \\
        craft 6 andesite slab using 3 andesite \\
        craft 6 andesite wall using 6 andesite \\ 
        craft 6 cobblestone wall using 6 cobblestone \\
        craft 2 diorite using 2 quartz, 2 cobblestone \\
        craft 4 andesite stairs using 6 andesite \\
        craft 1 granite using 1 diorite, 1 quartz \\
        craft 1 mossy cobblestone using 1 cobblestone, 1 vine \\
        craft 1 quartz block using 4 quartz \\
        craft 1 dropper using 1 redstone, 7 cobblestone \\
        craft 1 dispenser using 1 redstone, 7 cobblestone, 1 bow \\
        craft 4 diorite stairs using 6 diorite \\
        craft 2 andesite using 1 diorite, 1 cobblestone \\
        \textbf{Goal:} craft andesite wall.
        \vspace{1mm} \hrule \vspace{2mm}
        
        \textbf{[STEP 1]}\\
        \textbf{Action:} \\ \texttt{Inventory} \\
        \\
        \textbf{[STEP 2]}\\
        \textbf{Observation:} \\
        \texttt{Inventory: You are not carrying anything.} \\
        \textbf{Action:} \\ 
        \texttt{get 2 diorite} \\
        \\
        \textbf{[STEP 3]}\\
        \textbf{Observation:} \\
        \texttt{Could not find diorite.} \\
        \textbf{Action:} \\ \texttt{Get 2 quartz} \\
                \\
        \textbf{[STEP 4]}\\
        \textbf{Observation:} \\
        \texttt{Got 2 quartz.} \\
        \textbf{Value Function: 0.191} \\
        \textbf{Action:} \\ \texttt{Get 2 cobblestone}

        \dots \\ \\
        \textbf{[STEP 20]}\\
        \textbf{Observation:}  \\
        \texttt{Crafted 2 minecraft:andesite.} \\
        \textbf{Action:} \texttt{get 1 diorite} \\ \\
        \hrule \vspace{1mm}
        \textbf{Environment Reward: 0.0 (Failure)}
        
    \end{tcolorbox}
\end{minipage}
\label{fig:example_textcraft}
\end{figure}


\end{document}